%% file: main.tex
\documentclass[twoside,11pt]{article}
\usepackage{blindtext}
\usepackage[preprint]{style/jmlr2e}
\usepackage{xcolor}
\usepackage{tikz}
\usepackage{colortbl}

\usepackage{style/support-caption}
\usepackage{subcaption}
\usepackage[ruled,lined]{algorithm2e}
\usepackage{amsmath}
\usepackage{diagbox}
\usepackage{dsfont}
\usepackage[flushleft]{threeparttable}
\usepackage{booktabs}
\usepackage{amsfonts}
\usepackage{stmaryrd}
\newcommand{\mathbi}[1]{\boldsymbol{#1}}

\newcommand{\toolname}{xNeuSM\xspace}
\newcommand{\sstitle}[1]{\smallskip\noindent\textbf{#1.\/}}

\newtheorem{dfn}{Definition}

\newtheorem{problem}{Problem}

\usepackage{lastpage}
\jmlrheading{23}{2022}{1-\pageref{LastPage}}{1/21; Revised 5/22}{9/22}{21-0000}{Author One and Author Two}

\ShortHeadings{\toolname}{Explainable Neural Subgraph Matching}
\firstpageno{1}

\begin{document}

\title{xNeuSM: Explainable Neural Subgraph Matching with Graph Learnable Multi-hop Attention Networks}

\author{\name Duc Q. Nguyen \email nqduc@hcmut.edu.vn \\
       \addr Ho Chi Minh City University of Technology (HCMUT)\\
       Vietnam National University Ho Chi Minh City\\
       Ho Chi Minh City, Vietnam
       \AND
       \name {Thanh Toan Nguyen\footnote{Corresponding author}} \email thanhtoan.nguyen@griffithuni.edu.au \\
       \addr Griffith University\\
       Queensland, Australia
       \AND
       \name Tho Quan \email qttho@hcmut.edu.vn \\
       \addr Ho Chi Minh City University of Technology (HCMUT)\\
       Vietnam National University Ho Chi Minh City\\
       Ho Chi Minh City, Vietnam
       \AND
}

\editor{Daniel Hsu}
\maketitle

\begin{abstract}
Subgraph matching is a challenging problem with a wide range of applications in database systems, biochemistry, and cognitive science. It involves determining whether a given query graph is present within a larger target graph. Traditional graph-matching algorithms provide precise results but face challenges in large graph instances due to the NP-complete problem, limiting their practical applicability. In contrast, recent neural network-based approximations offer more scalable solutions, but often lack interpretable node correspondences. To address these limitations, this article presents \toolname: E\textbf{x}plainable \textbf{Neu}ral \textbf{S}ubgraph \textbf{M}atching which introduces Graph Learnable Multi-hop Attention Networks (GLeMA) that adaptively learns the parameters governing the attention factor decay for each node across hops rather than relying on fixed hyperparameters. We provide a theoretical analysis establishing error bounds for GLeMA's approximation of multi-hop attention as a function of the number of hops. Additionally, we prove that learning distinct attention decay factors for each node leads to a correct approximation of multi-hop attention. Empirical evaluation on real-world datasets shows that xNeuSM achieves substantial improvements in prediction accuracy of up to 34\% compared to approximate baselines and, notably, at least a seven-fold faster query time than exact algorithms. The source code of our implementation is available at \url{https://github.com/martinakaduc/xNeuSM}.
\end{abstract}

\begin{keywords}
  subgraph matching, subgraph isomorphism, graph neural networks, explainability, learnable multi-hop attention
\end{keywords}

\input{sections/introduction}
\input{sections/background}
\input{sections/approach}
\input{sections/experiment}
\input{sections/related}
\input{sections/conclusion}


\acks{This research is funded by Vietnam National Foundation for Science and Technology Development (NAFOSTED) under grant number IZVSZ2.203310.\\
Duc Q. Nguyen was funded by the Master, PhD Scholarship Programme of Vingroup Innovation Foundation (VINIF), code VINIF.2022.ThS.023.
}


\newpage









\input{sections/appendices}

\newpage
\bibliography{reference}

\end{document}

%% file: sections/introduction.tex
\section{Introduction} 
\label{sec:introduction}

In recent decades, a significant focus has been developing practical solutions for NP-hard graph problems. This wave of interest has been motivated by the abundance of diverse graph data in the public domain~\citep{sahu2020ubiquity}. One prominent challenge in this domain is tackling large graphs, and a vital aspect of this is addressing the issue of \emph{subgraph matching}. Essentially, subgraph isomorphism or subgraph matching involves determining whether a given query graph is isomorphic to a subgraph within a target graph. Despite the inherent NP-completeness, this problem holds paramount significance, as it applies to various domains, including social network analysis~\citep{fan2012graph}, bioinformatics~\citep{cannataro2012data}, and graph retrieval~\citep{roy2022interpretable}. This acceleration has driven researchers to devise scalable and efficient algorithms tailored to analyze extensive graphs like those found in social and biological networks. Throughout the past decades, numerous works have been developed, yielding a spectrum of practical solutions encompassing various algorithms~\citep{shang2008taming, han2013turboiso, bi2016efficient, bhattarai2019ceci, han2019efficient, lou2020neural}.

 \begin{figure}[!ht]
  \centering
  \includegraphics[width=0.8\linewidth]{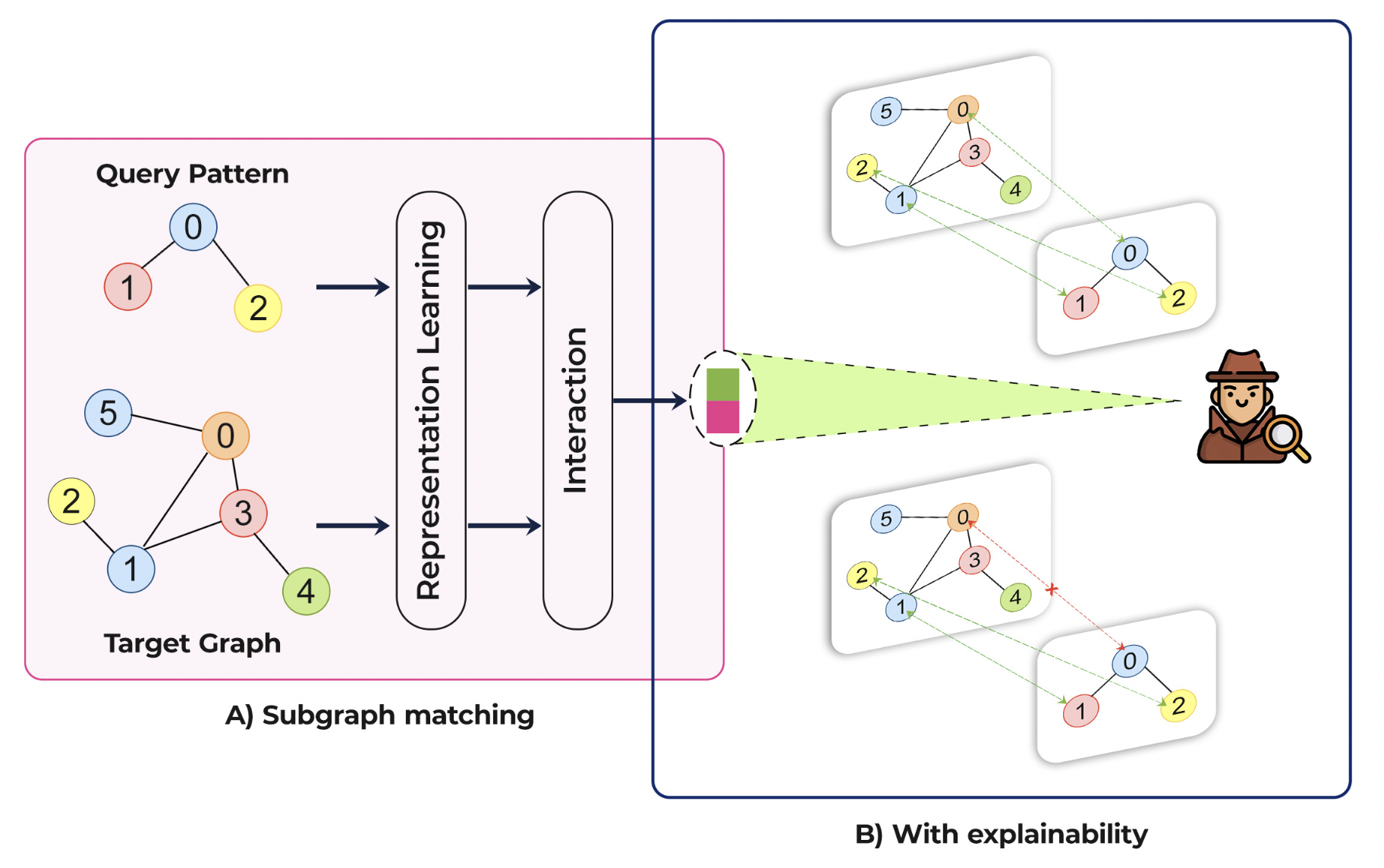}
  \vspace{-0.2cm}
  \caption{A typical subgraph matching result.}
  \label{fig:motivation}
\end{figure}

The conventional approaches~\citep{he2008graphs, han2013turboiso} rely on exact combinatorial search algorithms. While exact computation of subgraph matching yields precise results in both \emph{subgraph matching} (\autoref{fig:motivation}-A) and \emph{matching explanation} (finding node-to-node correspondences) among nodes across the two graphs (\autoref{fig:motivation}-B), they confront with limited \emph{scalability} when applied to larger query pattern sizes, which is attributed to the NP-complete nature of the problem. Recent efforts~\citep{bi2016efficient, bhattarai2019ceci, han2019efficient} aimed to enhance scalability by devising efficient matching orders and formulating powerful filtering strategies to shrink the number of candidates within the target graph. Although these efforts enable matching to be extended to larger target graphs, the query size remains restricted to just a few tens of nodes; this scalability level falls below what is needed for practical applications.

Subgraph matching is paramount due to its wide-ranging applications, but it poses a substantial challenge. In response to this challenge, neural-based approaches~\citep{xu2019cross, lou2020neural} have been proposed. These approaches aim to strike a balance between speed and accuracy. They demonstrate that by training a matching function to approximate the matching metric, it becomes feasible to identify candidate matches for a query pattern more rapidly than traditional combinatorial methods. The key innovation in these approaches lies in the use of graph neural networks (GNNs) to learn the matching function. However, these learning algorithms heavily rely on first-order dependencies within each layer of the GNN architecture, as highlighted in~\citep{lou2020neural}. This implies that the receptive field of a single GNN layer is limited to one-hop network neighbours, which are immediate neighbours in the graph. Nevertheless, recent research has revealed that data obtained from various complex systems may exhibit dependencies that extend beyond the first order, going as far as fifth-order dependencies~\citep{xu2016representing}. This contrasts with the earlier assumption of solely first-order network relationships. Oversimplifying the assumption of first-order network dependencies may lead to neglecting the generalizability of these methods in fully capturing patterns of varying sizes. Consequently, this oversimplification can result in a significant drop in performance across different domains~\citep{xu2019cross, lou2020neural}.


Balancing efficiency (\textbf{O.1}) and explanation (\textbf{O.2}) in a subgraph matching algorithm are desirable goals. However, developing a learning-based approach that achieves both objectives poses a significant challenge (see~\autoref{sec:solution_design} for a detailed discussion). Inspired by the recent success of Graph Multi-hop Attention Networks (GMA)~\citep{wang2020multi}, we have adopted the concept of GMA to address these two objectives simultaneously. Unlike existing multi-hop attention mechanisms in GMA that rely on a fixed attention decay factor for all nodes, this work introduces an innovative variant capable of adaptively learning this factor in a node-specific manner. This variant, termed Graph Learnable Multi-hop Attention Networks (GLeMa), parameterizes distinct attention decay factors for each node to govern their contributions across neighbourhoods during multi-hop message passing. Furthermore, we provide a theoretical analysis establishing approximation error bounds for GLeMa's modelling of multi-hop attention as a function of the number of hops. Additionally, we formally prove that learning node-specific attention decays enables GLeMa to capture relational patterns in graph-structured data accurately.

By incorporating the learnable multi-hop attention mechanism, GLeMA achieves a commendable level of generalization within the data graph while maintaining efficiency (achieving \textbf{O.1}). However, simultaneously achieving \textbf{O.1} and \textbf{O.2} is still challenging. Unlike previous neural-based approaches~\citep{lou2020neural} that directly learn from separate adjacency matrices for the pattern and target graphs, our newly devised unified proxy inputs facilitate the comprehensive capture of both intra- and inter-relations between the pattern and the target graph. This, in turn, enhances the explanation regarding explicit node alignment (achieving \textbf{O.2}). Additionally, we optimize the tasks of subgraph matching and matching explanation concurrently in an end-to-end multi-task manner. This approach leads to a mutually reinforcing synergy between both tasks, contributing to the overall effectiveness and efficiency of the framework. \emph{To the best of our knowledge, none of the existing learning-based methods have been able to accomplish both of these objectives simultaneously}.

We named our approach \toolname, E\textbf{\underline{x}}plainable \textbf{\underline{Neu}}ral \textbf{\underline{S}}ubgraph \textbf{\underline{M}}atching with Graph Learnable Multi-hop Attention Networks. Our contributions are stated as follows:

\begin{itemize}
    \item \emph{Graph Learnable Multi-hop Attention Networks:} We introduce GLeMa, which possesses the ability to directly learn node-specific attention decay factors from the data. This adaptable mechanism mitigates biases that could emerge from fixed attention decay mechanisms applicable universally across nodes, thus averting potential suboptimal outcomes. Our approach guarantees that the model's decisions are data-driven, steering clear of potential influences stemming from suboptimal parameter selections.
    
    \item \emph{Theoretical justifications for the approximation error of multi-hop attention and the correctness of node-specific attention decay factors:} In addition to the extensive empirical results, we conduct a theoretical analysis to approximate the error in multi-hop attention computation. Furthermore, we offer theoretical proof regarding the learning of attention decay factors specific to nodes. Our demonstration illustrates that the utilization of distinct decay factors remains consistent with the approximation of multi-hop attention and represents a generalized version of the previous universal attention decay factor utilized in earlier GMA models.

    \item \emph{Multi-task Learning:} We optimize both the subgraph matching and matching explanation tasks simultaneously within an end-to-end multi-task learning framework. This strategy fosters a mutually reinforcing synergy between these tasks, ultimately enhancing the overall effectiveness and performance of our framework.
\end{itemize}

The subsequent sections delineate the organization of this paper. Section \ref{sec:background} furnishes the foundational background of the study, setting the contextual framework. Section \ref{sec:approach} introduces our framework and succinctly outlines its constituent components. In Section \ref{sec:input}, we propose a method to construct a joint representation for both the pattern and target graphs, emphasizing the facilitation of intra-graph and inter-graph relationships within these newly formulated representations. Section \ref{sec:embedding} elaborates extensively on our proposed GLeMA, explicitly designed for learning graph representation and inter-graph interactions. In Section \ref{sec:multi_task}, we elucidate the aggregation of node embeddings to address subgraph matching and matching explanation tasks concurrently through a novel objective function. The evaluation of our approach using public datasets across various domains is presented in Section \ref{sec:experiment}. Notably, our proposed framework demonstrates a remarkable improvement, achieving more than a tenfold increase in subgraph matching speed compared to the fastest baseline and showcasing a notable enhancement of 27\% in accuracy and 34\% in F1-score compared to NeuroMatch, the state-of-the-art approximation approach, on the COX2 dataset. Additionally, our results exhibit comparability with exact methods. Section \ref{sec:related} delves into a discussion of related works, contextualizing our contributions within the existing literature. Conversely, Section \ref{sec:conclusion} culminates the paper with a concise summary.

%% file: sections/background.tex
\section{Background}
\label{sec:background}

In this section, we establish precise notations necessary preliminaries and formally define our targeted problem involving subgraph matching and matching explanation.

\subsection{Preliminaries}

In this study, we centre on solving subgraph matching and matching explanation problems on labelled, undirected, and connected graphs. Nevertheless, our proposed framework readily accommodates extensions to encompass directed graphs. The notations employed throughout this study are succinctly outlined and summarized in~\autoref{tlb:symbols}.

\begin{table}[!ht]
\centering
\caption{Summary of notation used}
\label{tlb:symbols}
    \begin{tabular}{c l}
    \toprule
    \textbf{Symbol} & \textbf{Definition} \\
    \midrule
            $\mathcal{D}$ & a set of data graphs\\
            $\mathcal{T}$ & the target graph ($T \in D$)\\
            $\mathcal{P}$ & the query pattern\\
            $\mathbi{A}^{in}$ & proxy intra-graph adjacency matrix\\
            $\mathbi{A}^{cr}$ & proxy cross-graph adjacency matrix\\
            $x^\ell_i$  & embedding of node $v_i$ at $\ell$-layer\\
            $\mathbi{X}^\ell$ & all node embeddings at $\ell$-layer\\
            $\mathcal{A}^{(1)}$ & $1$-hop attention matrix \\
            $\mathcal{A}$ & attention diffusion matrix\\
            $\mathcal{A}^{(K)}$ & approximate $K$-hop attention diffusion matrix \\
    \bottomrule
    \end{tabular}
\end{table}


Subsequently, we present formal definitions for a labelled, undirected, connected graph in Definition~\ref{def:labelled_undirected_graph} and extend this to a directed graph in Definition~\ref{def:labelled_directed_graph}. Building upon these definitions, we introduce the problem of induced labelled subgraph isomorphism in Definition~\ref{def:isomorphism}. 
In this study, we focus on induced subgraph isomorphism. After all, it is more difficult to solve than the non-induced one because it has fewer polynomial-time solvable exceptional cases~\citep{SYSLO198291}.

\begin{dfn}[\textbf{Labelled Undirected Connected Graph}]
\label{def:labelled_undirected_graph}
A labelled undirected connected graph is a graph represented with a 3-tuple $\mathcal{G}=(V,E,l)$ where
\begin{enumerate}
    \item $V$ is a set of nodes, 
    \item $E \subseteq [V]^2$ is a set of edges $(u, v)$, where $u,v \in V$
    \item $\forall v \in V, \text{deg}(v) \geq 1$
    \item $l: V \rightarrow \Sigma$ is a labelling function and $\Sigma$ is a set of node labels.
\end{enumerate}
\end{dfn}

\begin{dfn}[\textbf{Labelled Directed Connected Graph}]
\label{def:labelled_directed_graph}
A labelled directed connected graph is a graph represented with a 3-tuple $\mathcal{G}=(V, E,l)$ where
\begin{enumerate}
    \item $V$ is a set of nodes, 
    \item $E \subseteq [V]^2$ is a set of edges $(u, v)$, where $u$ is tail node, $v$ is head node and $u,v \in V$
    \item $\forall v \in V, (\text{deg}_{in}(v) \geq 1) \lor (\text{deg}_{out}(v) \geq 1)$
    \item $l: V \rightarrow \Sigma$ is a labelling function and $\Sigma$ is a set of node labels
\end{enumerate}
\end{dfn}


\begin{dfn}[\textbf{Induced Subgraph Isomorphism}]
\label{def:isomorphism}
Given two graphs $\mathcal{P} = (V_{\mathcal{P}}, E_{\mathcal{P}}, l_{\mathcal{P}})$ and
$\mathcal{T} = (V_{\mathcal{T}}, E_{\mathcal{T}}, l_{\mathcal{T}})$, $\mathcal{P}$ is considered as \emph{induced subgraph isomorphic} to $\mathcal{T}$ if there exists a label-preserving bijection $f:
V_{\mathcal{P}} \rightarrow V_{\mathcal{T}}$:
\begin{enumerate}
    \item $\forall \ v \in V_{\mathcal{P}}: l_{\mathcal{P}}(v) = l_{\mathcal{T}}(f(v))$, and
    \item $\forall \ (v_1,v_2) \in E_{\mathcal{P}}: (f(v_1), f(v_2)) \in 
    E_{\mathcal{T}}$, and
    \item $\forall v_1, v_2 \in V_{\mathcal{P}}: (f(v_1), f(v_2)) \in E_{\mathcal{T}} \Rightarrow (v_1, v_2) \in E_{\mathcal{P}}$.
\end{enumerate}
\end{dfn}

\subsection{Problem statement}


Our study concentrates on resolving two pivotal problems: subgraph matching and matching explanation, formally delineated in Problem \ref{prob:2} and Problem \ref{prob:3}, respectively.

\begin{problem}[\textbf{Subgraph Matching}]
\label{prob:2}
Given a target graph $\mathcal{T}$ and a query pattern $\mathcal{P}$, both labelled connected graphs, the subgraph matching problem aims to determine whether $\mathcal{P}$ is isomorphic to an induced subgraph of $\mathcal{T}$ or not.
\end{problem}

\begin{problem}[\textbf{Matching Explanation}]
\label{prob:3}
Given two graphs $\mathcal{P}=(V_{\mathcal{P}},E_{\mathcal{P}},l_{\mathcal{P}})$ and $\mathcal{T}=(V_{\mathcal{T}},E_{\mathcal{T}},l_{\mathcal{T}})$ where $\mathcal{P}$ is a known induced subgraph of $\mathcal{T}$, the matching explanation problem aims to accurately determine a one-to-one mapping that captures the node correspondences between $\mathcal{P}$ and $\mathcal{T}$. The objective is to establish precise relationships between nodes in $\mathcal{P}$ and their corresponding counterparts in $\mathcal{T}$.
\end{problem}


%% file: sections/approach.tex
\section{Overview approach}
\label{sec:approach}

\subsection{Design principles}
\label{sec:solution_design}
\toolname benefits from a crucial advantage in terms of efficiency, thanks to the inherent characteristics of neural network computation. However, when it comes to effectiveness, it is essential to carefully design the graph neural network architecture to meet the following three properties, in addition to the common ones like approximate accuracy and efficiency:

\begin{itemize}

\item \emph{\textbf{(R1) Explanibility.}} Ideally, a subgraph matching framework should be capable of identifying the pattern's presence and providing approximate \emph{alignment witnesses}. Given that no existing neural-based approaches offer these characteristics, we prioritize this feature as the utmost property due to its critical importance in numerous real-world applications.  

\item \emph{\textbf{(R2) High-order dependency.}} Conventional network representations, which implicitly assume the Markov property (first-order dependency), can swiftly become constraining. The oversimplification inherent in first-order networks may disregard scalability, particularly pattern size. Recent studies have demonstrated that data from numerous complex systems may exhibit dependencies as high as fifth-order~\citep{xu2016representing}. As we strive for a scalable solution in subgraph matching with explicability, including high-order dependency representation emerges as an essential design principle.

\item \emph{\textbf{(R3) Multi-task with configurability.}} The model should demonstrate adaptability to various matching metrics by fine-tuning its parameters through training. This is essential because, in certain scenarios, closely matched patterns~\citep{sussman2019matched}---those with a matching score surpassing a predetermined threshold---hold even greater significance than exact matches. Take, for instance, vaccine development, where a candidate closely matching the disease-to-be is far more critical than an exact match to the disease pattern. Such a closely matched candidate aids in early disease response. Hence, there may arise situations necessitating an emphasis on configuring the model to prioritize emergency scenarios that align better with human intuition~\citep{jimenez2020drug}.

\end{itemize}

\subsection{The challenges}
To accomplish these objectives, we must address the ensuing the following challenges:

\begin{itemize}
\item \emph{\textbf{Explainable adaptivity.}} The neural-based approach for the subgraph matching problem employs coarse-grained embeddings of entire graphs to approximate graph-level similarities (\textbf{R1}). Achieving an explicitly explainable alignment between nodes necessitates fine-grained annotations between the two graphs, adding a layer of supervision to the training process. However, preparing such training data is highly labour-intensive, resulting in computationally inefficient procedures.

\item \emph{\textbf{High computational complexity with Graph Multi-hop Attention.}} Addressing (\textbf{R2}) necessitates the effective integration of high-order dependencies, a task that is far from trivial. Elevating the orders of dependency can impose a computational burden on the model. As a result, many existing works frequently maintain the dependency fixed at the second order. Striking a balance between efficiency and capturing a large receptive field proves to be a challenging attempt.  

\item \emph{\textbf{Multi-objective optimization.}} Previous studies~\citep{scarselli2008graph,bai2018convolutional,zhang2019deep,lou2020neural} have employed neural networks to characterize the similarity function. These networks operate on graph-level embeddings or sets of node embeddings. Despite their competitive performance in approximating similarity and facilitating retrieval tasks compared to traditional computations, integrating an extra optimization objective for analyzing node-to-node mappings between query-target graph pairs (\textbf{R3}) presents challenges. The development of a neural model capable of seamlessly incorporating new objectives like the aforementioned one poses a significant challenge in architectural design.

\end{itemize}

\subsection{The approach}
We present a comprehensive overview of our \toolname framework in Figure \ref{fig:xSM_architecture}, delineating three primary stages as follows.

\sstitle{Input representation} Departing from previous neural-based methods~\citep{lou2020neural} that directly learn from separate adjacency matrices of pattern and target graphs, our newly devised unified proxy inputs enable the capture of cross-graph relations (\textbf{R1}). These inputs also bolster the learning process concurrently. Further elucidation on this aspect can be found in Section~\ref{sec:input}.
    
\sstitle{Graph Learnable Multi-hop Attention Networks} Our novel approach employs a specialized graph neural network, GLeMA, to extract higher-order dependencies. GLeMA facilitates the effective representation of inter- and intra-interactions between pattern and target graphs by a learnable multi-hop attention mechanism, enabling simultaneous learning of those interactions. This integration of high-order dependencies ensures scalability (\textbf{R2}), particularly with larger patterns. For further details, refer to Section~\ref{sec:embedding}.
    
\sstitle{Multi-task optimization} In this third stage, we aggregate node embeddings while concurrently addressing two tasks: \emph{subgraph prediction} and \emph{matching explanation}. Both tasks utilize the features learned in the preceding stage. Additionally, we introduce a novel objective function aimed at optimizing both tasks simultaneously (\textbf{R3}). Section~\ref{sec:multi_task} presents a detailed exploration of this stage.

\begin{figure}[!ht]
    \centering
    \includegraphics[width=\textwidth]{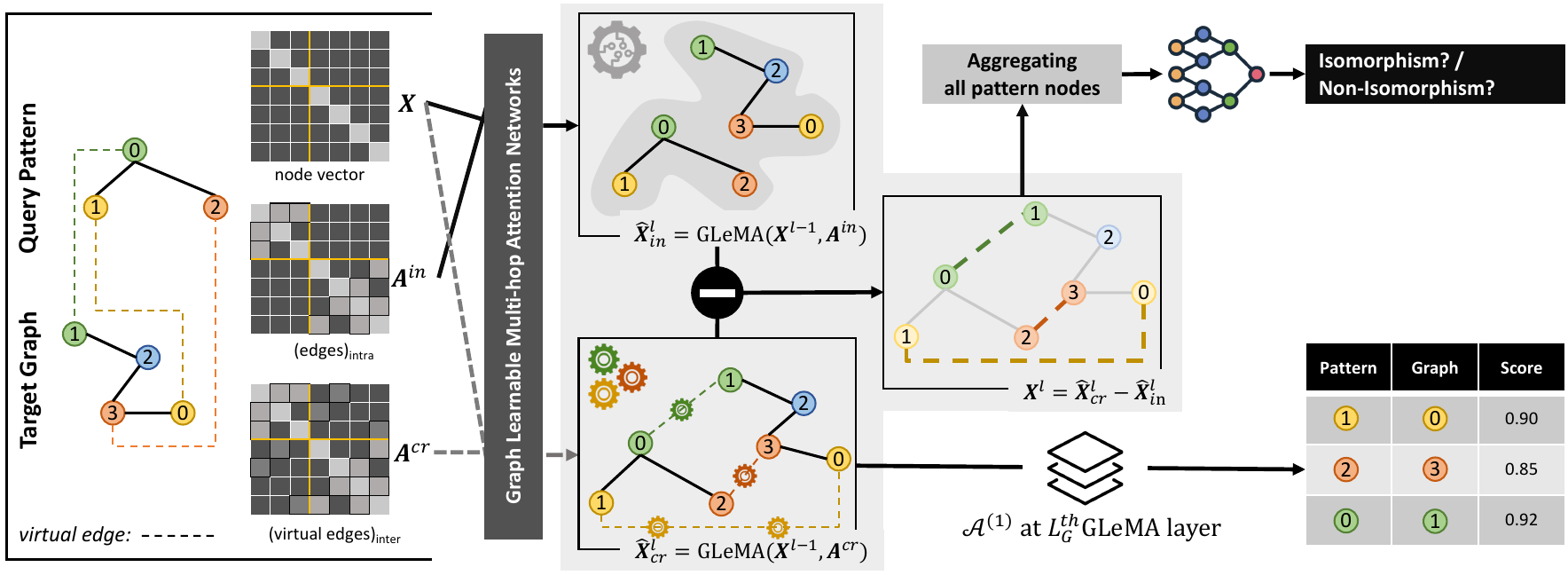}
    \caption{Overview of \toolname framework}
    \label{fig:xSM_architecture}
\end{figure}

\section{Input representation}
\label{sec:input}

Initially, the initialization of the model necessitates the preparation of input data. In the context of the subgraph isomorphism problem, the input comprises a subgraph (a pattern) and a larger graph (a target). Sets of nodes and edges conventionally characterize both patterns and targets. Toward our problem, we consider the pattern as $\mathcal{P} = \{V_{\mathcal{P}}, E_{\mathcal{P}}, l_{\mathcal{P}}\}$ and the target as $\mathcal{T} = \{V_{\mathcal{T}}, E_{\mathcal{T}}, l_{\mathcal{T}}\}$ where $V, E$ are the sets of nodes and edges respectively; $l: V \to T_V$ is the labelling function. Subsequently, the formulation of input for our proposed model is undertaken, encompassing a collection of node feature vectors denoted as $x$, the primary adjacency matrix represented by $\mathbi{A}^{in}$,  and the secondary adjacency matrix represented by $\mathbi{A}^{cr}$. Each node within the pattern or target is encoded as a one-hot vector of $2T_V$-dimensions, where $T_V = T_{V_{\mathcal{P}}} \cup T_{V_{\mathcal{T}}}$ stands for the maximum count of distinct node labels. The former $T_V$ dimensions are allocated for the pattern, while the remaining dimensions are designated for the target graph. This partitioning of pattern and target node features facilitates distinct embeddings for pattern and target nodes, thereby enhancing the quality of the mapping performance. Following this, the amalgamation of all node vectors culminates in forming the collective input set denoted as $\mathbi{X}$. The adjacency matrix $\mathbi{A}^{in}$ is created by flagging intra-graph edges, signifying the absence of edges connecting the pattern and target nodes. Conversely, the $\mathbi{A}^{cr}$ matrix considers a ``virtual'' edge connecting a pattern node and a target node in instances where they share identical labels. The mathematical definitions for $\mathbi{X}$, $\mathbi{A}^{in}$ and $\mathbi{A}^{cr}$ are formally expressed in equations (\ref{equ:gat_x_matrix}), (\ref{equ:gat_a1_matrix}), and (\ref{equ:gat_a2_matrix}), respectively.

\begin{align}
    \label{equ:gat_x_matrix}
    \mathbi{X} &= \{\vec{x}_1, \vec{x}_2, \dots, \vec{x}_{|V_{\mathcal{P}}|}, \vec{x}_{|V_{\mathcal{P}}| + 1}, \dots, \vec{x}_{|V_{\mathcal{P}}| + |V_{\mathcal{T}}|}\}\ \text{with}\ \vec{x}_i \in \mathbb{R}^{2T_V}\\
    \label{equ:gat_a1_matrix}
    \mathbi{A}^{in}_{ij} &= 
        \begin{cases}
            1 & \text{if there is an indirect edge or}\\
            & \text{a direct edge that connects $j$ to $i$}\\
            0 & \text{otherwise}
        \end{cases}\\
    \label{equ:gat_a2_matrix}
    \mathbi{A}^{cr}_{ij} &= 
        \begin{cases}
            \mathbi{A}^{in}_{ij} & \text{if $i, j \in \mathcal{P}$ or $i, j \in \mathcal{T}$}\\
            1 & \text{if $l(i) = l(j)$ and $i \in \mathcal{P}$ and $j \in \mathcal{T}$,}\\
            & \text{or if $l(i) = l(j)$ and $i \in \mathcal{T}$ and $j \in \mathcal{P}$}\\
            0 & \text{otherwise}
        \end{cases}
\end{align}



\section{Graph Learnable Multi-hop Attention Networks}
\label{sec:embedding}

In this section, we initially outline the process of extracting node features for the learning separation task, as detailed in Section~\ref{sec:node_features}. Following this, we introduce the workings of a single layer of the Learnable Multi-hop Attention mechanism in Section~\ref{sec:one_layer}. Subsequently, we discuss incorporating multiple layers of the attention mechanism to facilitate learning higher dependencies within the networks, as presented in Section~\ref{sec:multi_layer}.

\subsection{Extracting node features}
\label{sec:node_features}

In this study, we proposed a Graph Learnable Multi-hop Attention layer denoted as $\text{GLeMa}(\cdot)$, which utilizes Graph Attention Networks and a Learnable Multi-hop Attention mechanism. This approach facilitates the acquisition of holistic structural information of both targets and patterns. Assuming that we are applying this layer on an abstract graph $\mathcal{G} = \{V, E, l \}$, we present this graph with $(\mathbi{X}, \mathbi{A})$ where $\mathbi{X} \in \mathbb{R}^{|V| \times F}$ is the set of node features and $E$ is the set of edges. We formally denote the output of our proposed layer as $\widehat{\mathbi{X}} = \text{GLeMa}(\mathbi{X}, \mathbi{A})$. The input to our GLeMa layer encompasses a composite of node features, denoted as $\mathbi{X}$, and an adjacency matrix, denoted as $\mathbi{A}$, as described in equations \eqref{equ:gat_x} and \eqref{equ:gat_a}, respectively.

\begin{equation}
    \label{equ:gat_x}
    \mathbi{X} = \{\vec{x}_1, \vec{x}_2, \dots, \vec{x}_{|V|}\},\ \vec{x_i} \in \mathbb{R}^F,
\end{equation}
where $F$ is the number of features.
\begin{equation}
    \label{equ:gat_a}
    \mathbi{A}_{ij} = 
        \begin{cases}
            1 & \text{if there is an edge that connects $j$ to $i$}\\
            0 & \text{otherwise}
        \end{cases}
\end{equation}

Subsequently, node feature vectors are projected into embedding space by a linear transformation $\vec{x}{'}_i = \mathbi{W}_h \vec{x}_i$, $\vec{x}{'}_i \in \mathbb{R}^{F{'}}$, where $\mathbi{W}_h \in \mathbb{R}^{F{'} \times F}$ is a learnable weight matrix. Then, we calculate an attention coefficient for each pair of nodes by using Luong's attention \citep{luong-etal-2015-effective} as in \eqref{equ:attn_coef}.
\begin{equation}
    \label{equ:attn_coef}
    e_{ij} = 
        \begin{cases}
            \delta (\vec{x}{'}^T_i \mathbi{W}_e \vec{x}{'}_j) & \text{a directed edge $j$ to $i$},\\
            \delta (\vec{x}{'}^T_i \mathbi{W}_e \vec{x}{'}_j + \vec{x}{'}^T_j \mathbi{W}_e \vec{x}{'}_i) & \text{ an indirected edge $j$ to $i$},
        \end{cases}
\end{equation}


where $\delta$ is a non-linear activation function and $\mathbi{W}_e \in \mathbb{R}^{F{'} \times F{'}}$ is a learnable matrix. Subsequently, the normalization process involves subjecting all attention coefficients to the softmax function, resulting in the creation of the 1-hop attention matrix denoted as $\mathcal{A}^{(1)}$. Within this normalization procedure, the concept of \textit{masked attention} is incorporated, wherein solely the nodes $j \in \mathcal{N}_i$ are considered for the normalization operation. Here, $\mathcal{N}_i$ represents the set of neighbouring nodes of node $i$. Furthermore, to rigorously eliminate the influence of non-neighbour nodes, attention values normalized between two nodes $i$ and $j$ are replaced with zeros in instances where no edge is connecting node $i$ to node $j$. The mathematical framework for normalizing attention coefficients is articulated in \eqref{equ:attn_coef_norm}.

\begin{equation}
    \label{equ:attn_coef_norm}
    \begin{cases}
    \mathcal{A}^{(1)} &= \{ a^{(1)}_{ij} | i,j = \overline{1,|V|} \}\\
    a^{(1)}_{ij} &= \text{softmax}_j(e_{ij}) \mathbi{A}_{ij} = \frac{\text{exp}(e_{ij})}{\sum_{n \in \mathcal{N}_i} \text{exp}(e_{in})} \mathbi{A}_{ij}
    \end{cases}
\end{equation}

Upon obtaining the 1-hop attention matrix, we derive the attention diffusion matrix $\mathcal{A}$ following the concept of multi-hop mechanism in GNNs \citep{ijcai2021p425}. Formally, the matrix $\mathcal{A}$ is defined in \eqref{equ:attn_diff_mth}.

\begin{equation}
    \label{equ:attn_diff_mth}
    \begin{cases}
    (\mathcal{A}^{(1)})^{0} &= \mathbi{I}\\
    \mathcal{A} &= \sum_{k=0}^{\infty} \theta_k (\mathcal{A}^{(1)})^{k}\ \text{where}\ \sum_{k=0}^{\infty} \theta_k = 1\ \text{and}\ \theta_k > 0
    \end{cases}
\end{equation}




In \eqref{equ:attn_diff}, the parameter $\theta_k$ represents the attention decay factor, satisfying the condition $\theta_k > \theta_{k+1}$ to ensure a progressive reduction in importance for more distant nodes. Subsequently, for each node $i$, weighted summations are conducted between itself and other nodes to create a new feature vector for $i^{th}$ node, employing the multi-hop attention matrix.

Furthermore, the application of multi-head attention \citep{velickovic2018graph,ijcai2021p425} is executed to derive diverse feature representations from various distinct perspectives. The resultant vectors generated through multi-head attention are subsequently concatenated to yield the ultimate refined feature vectors for nodes. The formulation delineated in \eqref{equ:final_output} elucidates the process of generating the ultimate output within this layer.

\begin{equation}
    \label{equ:final_output}
    \mathbi{\widehat{X}} = \left( \bigparallel_{h=1}^{H} \delta \left( \mathcal{A}_{h} \mathbi{X{'}}_{h} \right) \right) \mathbi{W}_o\ \text{with}\ \mathbi{W}_o \in \mathbb{R}^{HF{'} \times F{'}}.
\end{equation}
        
In \eqref{equ:final_output}, the symbol $\mathbi{\widehat{X}}$ represents the collection of resultant node feature vectors, where $H$ signifies the count of attention heads employed in the multi-head attention mechanism. The term $\mathcal{A}_{h}$ denotes the multi-hop attention matrix associated with the $h$-th attention head. Correspondingly, $\mathbi{X{'}}_{h}$ designates the matrix representing projected node features for the $h$-th attention head, while $\mathbi{W}_o$ stands for a matrix of learnable weights.


\subsection{Learnable Multi-hop Attention mechanism}
\label{sec:one_layer}

The subsequent challenges confronting our research pertaining to (i) the elevated computational intricacies involved in computing $\mathcal{A}$ due to matrix multiplication \citep{gasteiger2018combining}, as well as (ii) the judicious selection of the suitable values for $\theta_k$, which significantly influences the augmentation or attenuation of the model performance \citep{gasteiger2018combining}.

\paragraph{\textbf{Reducing multi-hop attention matrix computation complexity}} In this study, following the methodology outlined in the previous GMA work \citep{ijcai2021p425}, we adopt the geometric distribution to determine $\theta_k$, wherein we choose $\theta_k = \alpha(1-\alpha)^k, \alpha \in (0,1)$ represents the teleport probability. Consequently, an approximation for $\mathcal{A} \mathbi{X{'}}$ is achieved via the utilization of equation \eqref{equ:attn_diff_approx}.

\begin{equation}
    \label{equ:attn_diff_approx}
    \begin{cases}
        \mathbi{Z}^{(0)} &= \mathbi{X{'}}\\
        \mathbi{Z}^{(k)} &= (1 - \alpha) \mathcal{A}^{(1)} \mathbi{Z}^{(k-1)} + \alpha \mathbi{Z}^{(0)},\ k = \overline{1,K}
    \end{cases}
\end{equation}

\begin{proposition}
$lim_{K \to \infty} \mathbi{Z}^{(K)} = \mathcal{A} \mathbi{X{'}}$
\end{proposition}

This proposition was proved in \citep{ijcai2021p425}, demonstrating that we can reduce the complexity of calculating $\mathcal{A}$ to $O(K|E|)$ message passing operators. Nonetheless, $K$ can be eliminated by considering $K$ as a constant and $K \ll + +\infty$. Doing so raises a concern about how well the $\mathcal{A}X^{'}$ is approximated by $Z^{(K)}$. In other words, the concern is selecting the $K$ that balances the trade-off between approximation error and computing complexity. The former works \citep{ijcai2021p425} suggested choosing $3 \le K \le 10$ by empirical experiments. In this study, we provide the theoretical justification for this selection strategy in Appendix \ref{sec:approx_err}. Consequently, assuming that $\mathcal{A}$ is approximated by $\mathcal{A}^{(K)}$ (Equation \ref{equ:attn_diff}), the complexity of computing $\mathcal{A}$ is reduced to just $O(|E|)$.

\begin{equation}
    \label{equ:attn_diff}
    \begin{cases}
    (\mathcal{A}^{(1)})^{0} &= \mathbi{I}\\
    \mathcal{A} &\approx \mathcal{A}^{(K)} = \mathbi{Z}^{(K)} \mathbi{X'}^{-1}
    \end{cases}
\end{equation}

\paragraph{\textbf{Learning teleport probabilities}} 
In graph attention diffusion theory, a critical concept is the definition of Personalized PageRank (PPR) \citep{lofgren2015efficient}, which reveals the importance of each node. The original work of GAT has proved that the attention matrix in GAT can be viewed as the transition matrix in PPR \citep{wang2020multi}. However, the original work used the same teleport probability for all nodes, which theoretically limits the power of PPR. Therefore, leveraging the capabilities of PPR, we propose a novel approach that involves the customization of teleport probabilities for individual nodes, denoted as $\beta = \{ \beta_v \}_{v=1}^{|V|}$. The primary challenge at hand revolves around selecting appropriate $\beta_v$ values for each node. To address this challenge, inspired by Gated Recurrent Units \citep{cho-etal-2014-learning}, we devise a method wherein the network autonomously learns the teleport probabilities via a straightforward linear transformation with $\mathbi{W}_{\beta} \in \mathbb{R}^{2F{'} \times 1}$. The equation presented in (\ref{equ:alpha_approx}) delineates how $\beta$ is derived.
\begin{equation}
    \label{equ:alpha_approx}
    \beta = \sigma((\mathbi{X{'}} || \mathcal{A}^{(1)}\mathbi{X{'}}) \mathbi{W}_{\beta} + b),
\end{equation}
where $||$ denotes the concatenation operator, and $b$ represents the bias term. It is worth noting that employing distinct teleport probabilities for each node results in modifications to Equation \ref{equ:attn_diff_mth}. These alterations encompass $\theta_{kj} = \beta_j(1-\beta_j)^k$, $\sum_{k=0}^{\infty} \theta_{kj} = 1$ for $j \in \overline{1,N}$, and $\theta_{kj} > 0$. Furthermore, they involve a row-wise multiplication between $\theta_k$ and $(\mathcal{A}^{(1)})^k$. It is important to emphasize that these changes do not contravene \textbf{Proposition 1}, as established in Appendix \ref{sec:separated_teleport}. Additionally, we provide a comprehensive procedure for computing learnable multi-hop attention in Algorithm \ref{alg:multihop}.

\begin{algorithm}[!ht]
\caption{Learnable Multi-hop Attention}
\label{alg:multihop}
\SetKwInOut{Input}{Input}
\SetKwInOut{Output}{Output}
    
\Input{1-hop attention matrix $\mathcal{A}^{(1)}$\\
       Node feature matrix $\mathbi{X'}$\\
       Number of approximate hops $K$
      }
\Output{Diffused node feature matrix $\mathbi{\widehat{X}}$}
$\mathbi{Z}^{(0)} \gets \mathbi{X'}$\\
$\beta \gets \sigma((\mathbi{X{'}} || \mathcal{A}^{(1)}\mathbi{X{'}}) \mathbi{W}_{\beta} + b)$\\
\For{k \text{\textbf{in}} \text{Range($1 \dots K$)}}
{
    $\mathbi{Z}^{(k)} = (1 - \beta) \mathcal{A}^{(1)} \mathbi{Z}^{(k-1)} + \beta \mathbi{Z}^{(0)}$
}
$\mathbi{\widehat{X}} \gets \mathbi{Z}^{(K)}$
\end{algorithm}

\subsection{Graph Learnable Multi-hop Attention Networks}
\label{sec:multi_layer}
When processing input data that comprises a pattern and a target, denoted as a triple $(\mathbi{X}, \mathbi{A}^{in}, \mathbi{A}^{cr})$, we employ our proposed GLeMa layer ($\text{GLeMa}(\cdot)$) to extract hidden features. Specifically, the input is bifurcated into two tuples: $(\mathbi{X}, \mathbi{A}^{in})$ and $(\mathbi{X}, \mathbi{A}^{cr}$), which are subsequently subjected to $L_G$ iterations of GLeMa layers. The resulting representation for $(\mathbi{X}, \mathbi{A}^{in})$ captures intra-graph features, whereas the representation for $(\mathbi{X}, \mathbi{A}^{cr}$) captures inter-graph features.

The node features at the $l^{th}$ layer are computed by taking the difference between the inter-graph features and the intra-graph features from the previous $(l-1)^{th}$ layer. This learning of disparities between inter-graph and intra-graph features enhances the signal for verifying subgraph isomorphism. The formal definition of the Graph Learnable Multi-hop Attention network architecture is presented in Equation \ref{equ:GLeMa_networks}.
\begin{equation}
\label{equ:GLeMa_networks}
\begin{cases}
    \mathbi{X}^{0} = \mathbi{X}\\
    \mathbi{\widehat{X}}_{in}^{l} = \text{GLeMa}_l(\mathbi{X}^{l-1}, \mathbi{A}^{in}), l=\overline{1, L_G}\\
    \mathbi{\widehat{X}}_{cr}^{l} = \text{GLeMa}_l(\mathbi{X}^{l-1}, \mathbi{A}^{cr}), l=\overline{1, L_G}\\
    \mathbi{X}^{l} = \mathbi{\widehat{X}}_{cr}^{l} - \mathbi{\widehat{X}}_{in}^{l}, l=\overline{1, L_G}
\end{cases}
\end{equation}

\section{Multi-task optimization}
\label{sec:multi_task}

In this section, we introduce the optimization mechanisms for both the subgraph matching task and the matching explanation task, discussed in Section~\ref{sec:task1} and Section~\ref{sec:task2}, respectively. Subsequently, we delve into the optimization approach for multi-task learning, as outlined in Section~\ref{sec:2tasks}.

\subsection{Subgraph matching task}
\label{sec:task1}
In the context of subgraph matching, after their extraction via $L_{G}$ GLeMa layers, the node feature vectors derived from patterns are aggregated to generate a representation vector. This representation vector is pivotal in determining the isomorphism between the input pattern and the target graph, achieved through a classifier comprising $L_{FC}$ fully connected layers. The methodology for computing the representation vector is elucidated in Equation \ref{equ:aggregation}, while Equation \ref{equ:fc_layers} provides the mathematical formulations underpinning the classifier.

\begin{equation}
    \label{equ:aggregation}
    x^{0}_{repr} = \frac{1}{|V_\mathcal{P}|}\sum_{i \in V_{\mathcal{P}}} x_i^{L_{G}}
\end{equation}

\begin{equation}
    \label{equ:fc_layers}
    \begin{cases}
    x^{i}_{repr} &= \delta (\mathbi{W}_i x^{i-1}_{repr} + b_i), i = \overline{1,L_{FC}-1}\\
    \widehat{y} &= \sigma (\mathbi{W}_y x^{L_{FC}-1}_{repr} + b_y)
    \end{cases}
\end{equation}
In Equation \ref{equ:aggregation}, we denote $x^{0}_{repr}$ as the representation vector of the input, and $x_i^{L{G}}$ as the embedding vector of the $i^{th}$ node after undergoing $L_{G}$ GLeMa layers. In \eqref{equ:fc_layers}, we represent $x^{i}_{\text{repr}}$ as the output vector of the $i$-th fully-connected layer, with $\mathbi{W}_i$ and $b_i$ signifying the respective learnable weight matrix and bias parameters.

\subsection{Matching explanation task}
\label{sec:task2}

Leveraging the efficacy of the multi-hop attention mechanism, our proposed model can predict the mapping of a pattern within a target graph. The enigmatic aspect of our approach entails filtering pairs of \textit{(pattern node, target node)} based on the 1-hop attention coefficients obtained from the final inter-graph GLeMa layer, subject to a predetermined threshold. Assuming that a threshold-compliant pair signifies a valid mapping between a pattern node and a target node, our model can enumerate all potential mappings, irrespective of whether the input pattern is isomorphic to the target or not. The precise computational details of this matching task are expounded upon in Equation \ref{equ:mapping}. In \eqref{equ:mapping}, $\mathcal{M}$ is the set of mapping nodes between the pattern and target graph; $p_{ij}$ which is computed by average of 1-hop attention coefficients $((a^{(1)}_{ij})^{L_{G}}, (a^{(1)}_{ji})^{L_{G}})$ is the mapping probability between pattern $i^{th}$ node and target $j^{th}$ node.

\begin{equation}
    \label{equ:mapping}
    \begin{split}
    \mathcal{M} =& \{ (i, j, p_{ij}) | p_{ij} \ge \epsilon \}, \text{where } i \in V_{\mathcal{P}}, j \in V_{\mathcal{T}} \text{ and}\\
    & p_{ij} = \frac{1}{2} \left((a^{(1)}_{ij})^{L_{G}} + (a^{(1)}_{ji})^{L_{G}}\right).
    \end{split}
\end{equation}

\subsection{Multi-task learning optimization}
\label{sec:2tasks}

In order to optimize our proposed models for the dual tasks of subgraph matching and matching explanation, we introduce a composite loss function consisting of two fundamental components. The first component, denoted as $\mathcal{L}{sm}$, is a binary cross-entropy loss designed to accurately assess the model's capacity to predict subgraph isomorphism. The second component, $\mathcal{L}{me}$, is an attention-based loss aimed at reinforcing the attention coefficients between nodes $i$ and $j$ ($i \in V_{\mathcal{P}}, j \in V_{\mathcal{T}}$) that correspond to actual mappings, while simultaneously diminishing the coefficients for node pairs sharing the same label (represented as $m \in V_{\mathcal{P}}, n \in V_{\mathcal{T}}, l(m)=l(n)$) but lacking a mapping relationship.

Our overarching objective function, as expressed in Equation \ref{equ:loss_fn}, comprises a combination of $\mathcal{L}_{sm}$ and $\lambda \mathcal{L}_{me}$, where the weight $\lambda$ serves as a hyperparameter for regulating the relative importance of the two loss components.

\begin{equation}
    \label{equ:loss_fn}
    \begin{cases}
    \mathcal{L}_{sm} &= -\frac{1}{|\mathcal{D}|} \sum_{k=1}^{|\mathcal{D}|} y_k \cdot \text{log}(\widehat{y}_k) + (1-y_k) \cdot \text{log}(1-\widehat{y}_k)\\
    \mathcal{L}_{me} &= \frac{1}{|\mathcal{D}|} \sum_{k=1}^{|\mathcal{D}|} \frac{\sum \text{exp}\left(-\left(a_{ij}^{(1)(L_{G})}\right)_k\right)}{\sum \text{exp}\left(-\left(a_{mn}^{(1)(L_{G})}\right)_k\right) - \sum \text{exp}\left(-\left(a_{ij}^{(1)(L_{G})}\right)_k\right) + 1}\\
    \mathcal{L} &= \mathcal{L}_{sm} + \lambda \mathcal{L}_{me}
    \end{cases}
\end{equation}

%% file: sections/experiment.tex
\section{Experiments}
\label{sec:experiment}
In this section, we provide comprehensive evaluation protocols designed for \toolname. Subsequently, we conduct rigorous experiments using six real-world datasets to assess the performance of \toolname. The primary objective is to address the following research questions (RQs) through systematic experimentation:

\begin{itemize}
\item \textbf{RQ1}: Does \toolname exhibit superior performance compared to diverse baseline techniques in the subgraph matching task?
\item \textbf{RQ2}: How confident are the predictions made by \toolname?
\item \textbf{RQ3}: What impact do variations in pattern size and density have on the performance of \toolname?
\item \textbf{RQ4}: How effectively does \toolname perform in the matching explanation task?
\item \textbf{RQ5}: What is the individual contribution and impact of each component of \toolname?
\item \textbf{RQ6}: Can \toolname effectively handle inductive settings, adapting to new, previously unseen graphs?
\item \textbf{RQ7}: Does \toolname retain its effectiveness when dealing with directed subgraph matching and explanation scenarios?
\end{itemize}

Each research question is carefully designed to delve into specific aspects of \toolname's performance, scalability, interpretability, and adaptability across varying conditions and settings, providing a nuanced understanding of its capabilities and limitations. The experiments conducted aim to provide robust empirical evidence in addressing the above research inquiries, proving the effectiveness of our framework in reality. 


\input{sections/exp_setup}
\input{sections/exp_end2end}

\input{sections/exp_confidence}
\input{sections/exp_scalability}
\input{sections/exp_explanation}
\input{sections/exp_abtest}

\input{sections/exp_generalisation}
\input{sections/exp_directed}

%% file: sections/exp_setup.tex
\subsection{Experimental setup}
\label{sec:exp:setup}
In this subsection, a detailed exposition of our experimental framework is presented, encompassing the selection of datasets, baseline methods, data pre-processing methodologies, and the evaluation metrics employed. This comprehensive account aims to elucidate the specifics of our experimental setup, ensuring transparency and replicability while providing clarity on the foundational elements shaping our assessment methodology.

\subsubsection{Datasets}
We assess the performance of our framework across a diverse range of real datasets encompassing various domains, including citation, collaboration, social, and communication networks. To evaluate the effectiveness of \toolname, we conduct experiments on six well-established real-world datasets frequently employed in various applications~\citep{Morris2020} within graph mining research. The statistics for these datasets are summarized in Table \ref{tab:dataset}.

\begin{table}[!ht]
    \centering
    \caption{Real datasets}
    \label{tab:dataset}
    \begin{tabular}{ccccccc}
        \toprule
         \textbf{Domain} & \textbf{$\mathcal{D}$} & \textbf{$\overline{|V(G)|}$} & \textbf{$\overline{|E(G)|}$} & \textbf{$\overline{D}$} & \textbf{$|\sum|$} & $T_V$ \\
         \midrule
         Bioinformatics & KKI & 26.96 & 48.42 & 3.19 & 83 & 190\\
         Chemistry & COX2 & 41.22 & 43.45 & 2.10 & 467 & 20\\
         Chemistry & COX2\_MD & 26.28 & 335.12 & 25.27 & 303 & 36\\
         Chemistry & DHFR & 42.43 & 44.54 & 2.10 & 756 & 71\\
         Social networks & DBLP-v1 & 10.48 & 19.65 & 3.43 & 19456 & 39\\
         Computer vision & MSRC-21 & 77.52 & 198.32 & 5.10 & 563 & 141\\
         \bottomrule
    \end{tabular}
\end{table}



\subsubsection{Baseline techniques}
We provide details of our baseline methods, including state-of-the-art exact and approximate approaches as follows.

\sstitle{Exact approach} We utilize seven distinct approaches as below.
\begin{itemize}
\item \emph{VF3}\citep{carletti2017introducing} is an extension of the VF2 algorithm\citep{cordella2004sub} designed to handle larger graphs. It employs an enhanced bit vector representation for graph states, introduces a novel matching order for query nodes, and incorporates various heuristics like degree-based filtering and dynamic node reordering to enhance matching efficiency.
\item \emph{TurboISO}~\citep{han2013turboiso} is an efficient subgraph isomorphism algorithm that utilizes pre-processing, constructing an index for rapid candidate subgraph filtering. It employs a divide-and-conquer strategy along with heuristics such as degree-based filtering and forward checking to narrow down the search space.
\item \emph{CFL}~\citep{bi2016efficient} minimizes redundant Cartesian products in the search space by strategically postponing them based on query structure. It introduces a compact path-based auxiliary data structure for efficient matching.
\item \emph{CECI}~\citep{bhattarai2019ceci} partitions the target graph into multiple clusters for parallel processing and uses BFS-based filtering, reverse-BFS-based refinement, and set intersection techniques to optimize the search process.
\item \emph{QuickSI}~\citep{shang2008taming} employs QI-Sequence to constrain the search space, determining order based on feature frequencies. It introduces Swift-Index, reducing costs in the filtering phase.
\item \emph{DAF}~\citep{han2019efficient} introduces concepts like dynamic programming between a directed acyclic graph (DAG) and a graph, adaptive matching order with DAG ordering, and pruning by failing sets to address limitations of existing algorithms.
\item \emph{GraphQL}~\citep{he2008graphs} introduces a specialized query language for graph databases, extending formal languages from strings to graphs. It uses neighbourhood subgraphs and profiles to optimize the search order and reduce the search space.
\end{itemize}

\sstitle{Approximate approach} Our comparison includes \emph{NeuralMatch}~\citep{lou2020neural}, a cutting-edge subgraph matching algorithm employing a specialized graph neural network architecture. This approach efficiently identifies the neighbourhood within a large target graph that encompasses a smaller query graph as a subgraph. Through a GNN, it acquires robust graph embeddings in an ordered space, encapsulating structural properties like transitivity, antisymmetry, and non-trivial intersections. As a result, NeuralMatch achieves real-time approximate subgraph matching at an unprecedented scale.

\subsubsection{Data preparation} 

In our experimental setup, distinct training and testing datasets are utilized. The testing dataset comprises real-world instances, while the training dataset is artificially synthesized to mirror the size and degree distribution of the respective testing dataset.

For each graph in the testing dataset $\mathcal{D}$, we generate 2000 queries, half of which are isomorphic to the graphs and vice versa. The sizes of the query graphs vary from 2 to the size of the data graph, adhering to a \textit{Uniform} distribution. The average degrees of the queries follow a Normal distribution $\mathcal{N}(\overline{D}, \sigma^2_{\mathcal{D}}(D))$.

Corresponding to each real dataset, we create a synthetic training dataset that replicates the graph size and degree distribution. We adopt an identical process to generate 2000 queries for each target graph, akin to the procedure employed in the testing set. The number of target graphs in these synthetic training datasets is four times greater than in the real datasets. 





\subsubsection{Metrics} 

\paragraph{\textbf{Subgraph matching task:}} In this task, we conduct a comparative analysis of our proposed approach by evaluating it against exact methods and approximate methods using various metrics to assess its runtime and performance comprehensively. The metrics utilized in this evaluation are as follows:

\begin{itemize}
\item \emph{Execution time}: emph{Execution time}: This metric denotes the average processing time for a query (target graph, query pattern), excluding disk loading time.

\item \emph{ROC AUC}: The ROC curve illustrates model performance by plotting the True Positive Rate (TPR) against the False Positive Rate (FPR) at various thresholds. The AUC quantifies overall performance based on the curve's area, ranging from 0 to 1. Higher values denote better performance.

\item \emph{PR AUC}: This metric uses Precision-Recall curves, emphasizing Precision and Recall rates over TPR and FPR. It quantifies the area under the Precision-Recall curve, providing a more accurate assessment of imbalanced datasets. Higher PR AUC values indicate better performance.

\item \emph{F1 score}: The F1 score, a harmonic mean of Precision and Recall, offers a balanced evaluation of the model's performance. Precision measures accuracy in identifying positive instances, while Recall assesses the ability to capture all actual positives.
    
\end{itemize}

\paragraph{\textbf{Matching explanation task:}} To demonstrate the efficacy of identifying mappings of isomorphic subgraphs, we extract the attention matrix from the last GLeMA layer in the branch that incorporates cross-graph connections. Subsequently, we rank the mappings of each query node within the transaction according to their respective attention scores. We then employ the following two metrics for evaluation:
    \begin{itemize}
        \item \emph{Average Top-$K$ Accuracy}: Assuming that $acc^K_i$ represents the Top-$K$ accuracy of node $i$ in the query, we compute the average Top-$K$ accuracy across all testing samples using the following equation.
        \[
        TopK = \frac{1}{|\mathcal{D}_{test}|} \sum_{ (\mathcal{T}, \mathcal{P}) \in \mathcal{D}_{test}} \left( \frac{1}{|V_{\mathcal{P}}|} \sum_{i \in V_{\mathcal{P}}} Acc^K_i \right)
        \]

        \item \emph{Mean Reciprocal Rank}: This metric assesses the model capability to predict the correct mapping with a high probability. It is computed by taking the multiplicative inverse of the rank of the first correct mapping. Let $rank_i$ represent the ranking of the correct mapping for query node $i$. The average reciprocal ranking across test samples can be calculated as follows.
        \[
        MRR = \frac{1}{|\mathcal{D}_{test}|} \sum_{ (\mathcal{T}, \mathcal{P}) \in \mathcal{D}_{test}} \left( \frac{1}{|V_{\mathcal{P}}|} \sum_{i \in V_{\mathcal{P}}} \frac{1}{rank_i} \right)
        \]
    \end{itemize}



%% file: sections/exp_end2end.tex
\subsection{Subgraph matching}

To address the first research question (\textbf{RQ1}), we conduct a comprehensive evaluation of \toolname's performance in an end-to-end manner, with a specific focus on several critical aspects: (i) execution time and (ii) performance in subgraph matching tasks, assessed through four distinct metrics which are ROC AUC, PR AUC, F1 score, and accuracy.

\sstitle{Execution Time across real datasets}
We present the findings in~\autoref{fig:e2e_time}. The results unequivocally demonstrate that our proposed solution, \toolname, boasts the shortest execution time when compared to all state-of-the-art exact and approximate methods. This underscores \toolname's superior scalability, enabling it to process larger graphs effectively.

\begin{figure}[!ht]
    \centering
    \includegraphics[width=\textwidth]{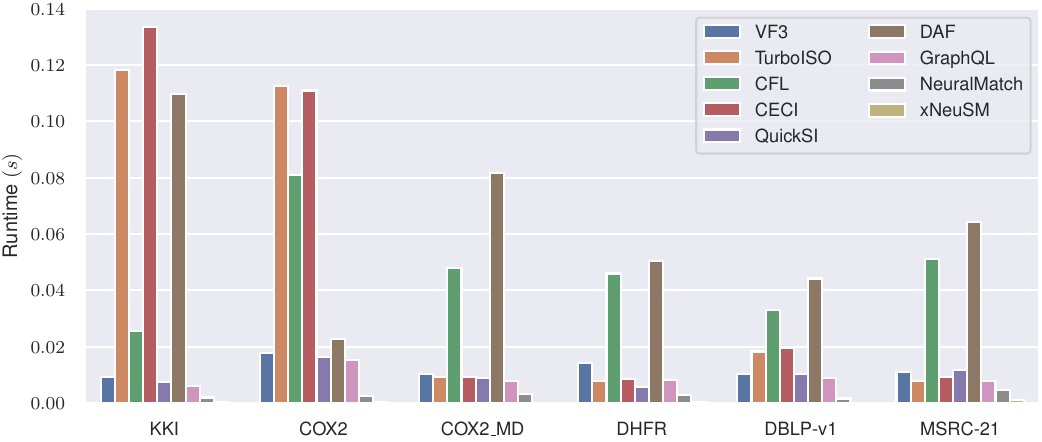}
    \caption{Execution time on subgraph matching task.}
    \label{fig:e2e_time}
\end{figure}

\sstitle{Performance across real datasets}
We showcase the benchmarking results of our \toolname and approximate approach in~\autoref{fig:e2e_performance}. The outcomes unequivocally affirm that our approach achieves performance levels nearly on par with exact methods while surpassing the current SOTA performance of the approximate method across all tested real datasets. This underscores the versatility of our approach, making it applicable in diverse real-world scenarios such as pattern matching in social networks, identifying compounds with specific activities, and beyond.

\begin{figure}[!ht]
    \centering
    \includegraphics[width=\textwidth]{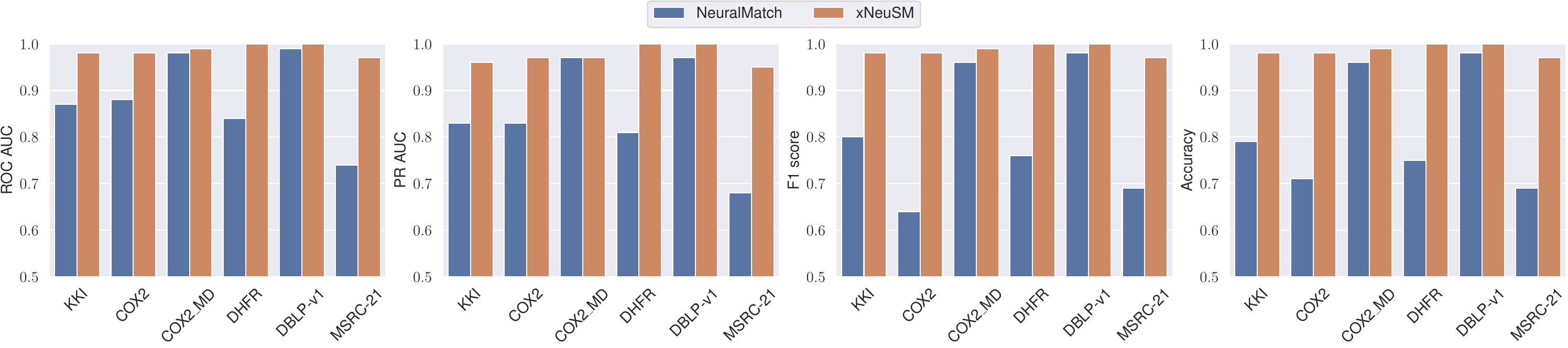}
    \caption{Performance comparison between Neural Match and \toolname}
    \label{fig:e2e_performance}
\end{figure}








%% file: sections/exp_confidence.tex
\subsection{Confidence Analysis}
In this section, we assess the model confidence in its predicted outputs to answer the \textbf{RQ2}. We demonstrate that our model maintains high-performance levels by elevating the output probability threshold for the subgraph matching task. Figure \ref{fig:confidence} depicts the relationship between confidence threshold and model performance. With a confidence threshold of $0.9$, our model attains over $90\%$ across all metrics in the testing datasets. This suggests that our \toolname exhibits robust confidence in its predictions.

\begin{figure}[!ht]
    \centering
    \includegraphics[width=\textwidth]{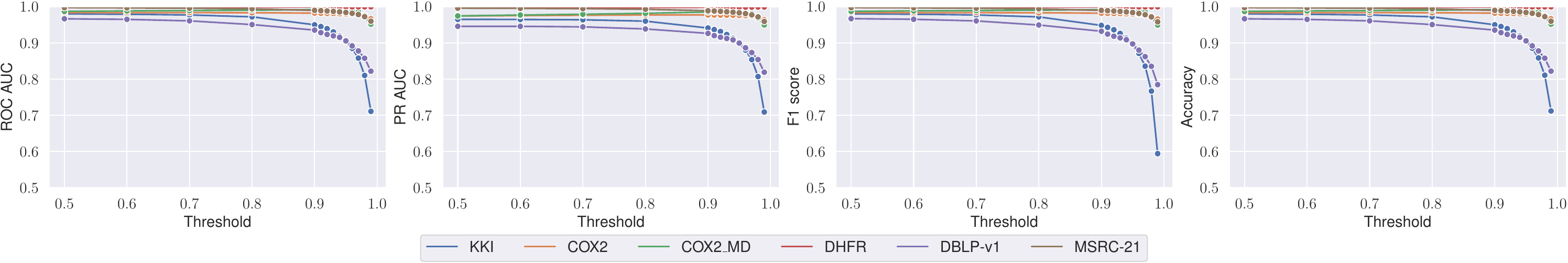}
    \caption{Relation between confidence threshold and model performance}
    \label{fig:confidence}
\end{figure}

%% file: sections/exp_scalability.tex
\subsection{Scalability}

To assess the scalability of \toolname\ and address research question \textbf{RQ3}, we evaluated the performance of all techniques using various real datasets with different levels of graph density, ranging from sparse to dense graphs.

\begin{itemize}
\item Vary $D(P)$: We divided queries into two subsets based on their average degree. The first subset, labelled ``dense", included queries with a degree of three or higher. The second subset, labelled ``sparse", encompassed queries with a degree less than 3.
\item Vary $|V_P|$: We partitioned the query set into four groups based on query size thresholds: $|V_P| \leq 20$, $20 < |V_P| \leq 40$, $40 < |V_P| \leq 60$, and $60 < |V_P|$.
\end{itemize}

    
    

We present our results in Figure \ref{fig:scale_time_by_dataset}, with runtime represented on a logarithmic scale. These results show that exact methods experience a significant increase in runtime when the number of query nodes is augmented. Specific methods, such as CECI or CFL, exhibit high sensitivity to the number of query nodes. In contrast, despite experiencing increased time requirements, our methods demonstrate relatively small increments due to the parallelizability of all operations using GPU. Consequently, our method proves to be more efficient in large-scale settings.

\begin{figure}[!ht]
    \centering
    \includegraphics[width=\textwidth]{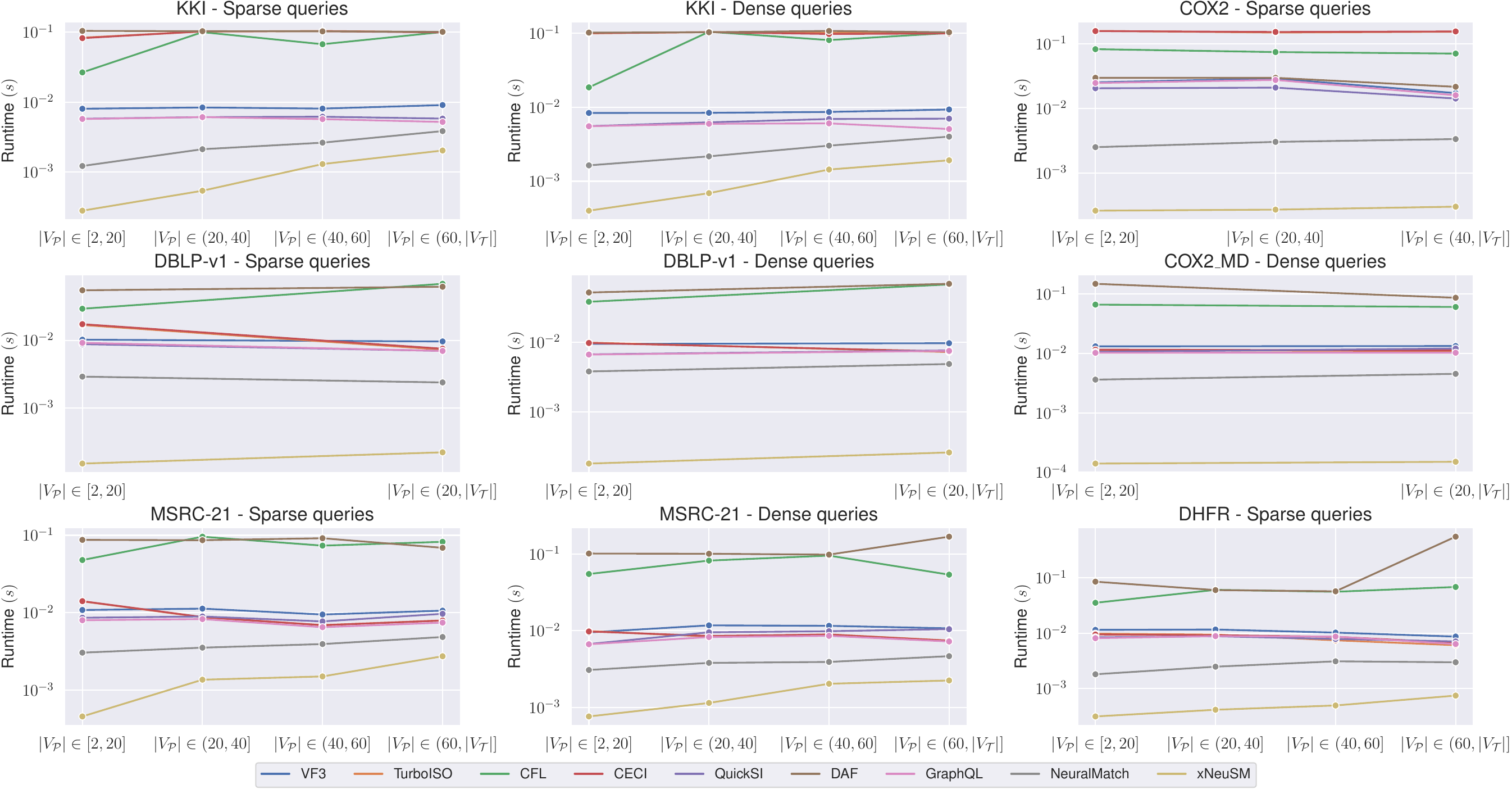}
    \caption{Runtime scalability analysis on six datasets including: KKI, COX2, COX2\_MD, DHFR, DBLP-v1 and MSRC-21}
    \label{fig:scale_time_by_dataset}
\end{figure}






%% file: sections/exp_explanation.tex
\subsection{Matching explanation}
\sstitle{Quantitative analysis}
We conducted experiments on the matching explanation task to address \textbf{RQ4}. It is important to note that this task exclusively applies to known isomorphism pairs of (pattern, target). Non-isomorphic cases were deliberately excluded from our testing as they may not represent real-world use cases. Subgraph mapping is necessary when the isomorphism is established. Within this task, we computed the attention scores for all transaction nodes concerning each query node within the inter-attention branch of the final GLeMA layer. Subsequently, we organized the list of corresponding transaction nodes for each query node based on the computed attention scores. The effectiveness of subgraph alignment is evaluated and presented in Table \ref{tab:explain_alinging}.

\autoref{tab:explain_alinging} illustrates the superior performance of our method in the subgraph alignment task across diverse datasets, such as KKI, DHFR, DBLP-v1, and MSRC-21. However, the task presents increased challenges when dealing with a limited set of node labels, which can complicate the retrieval of accurate subgraph mappings. This complexity arises from the exponential growth in the potential number of mappings.

\begin{table}[!ht]
    \caption{Performance of  in subgraph aligning task}
    \label{tab:explain_alinging}
    \centering
    \begin{tabular}{ccccc}
        \toprule
        \textbf{Dataset} & \textbf{Top-1}$\uparrow$ & \textbf{Top-5}$\uparrow$ & \textbf{Top-10}$\uparrow$ & \textbf{MRR}$\uparrow$\\
        \midrule
        KKI & 0.9978 & 0.9999 & 0.9999 & 0.9987\\
        COX2 & 0.2513 & 0.6259 & 0.8395 & 0.4273\\
        COX2\_MD & 0.9481 & 0.9828 & 0.9881 & 0.9630\\
        DHFR & 0.9999 & 0.9999 & 0.9999 & 0.9999\\
        DBLP-v1 & 0.9994 & 0.9999 & 0.9999& 0.9996\\
        MSRC-21 & 0.9994 & 0.9999 & 0.9999& 0.9999\\
        \bottomrule
    \end{tabular}
\end{table}

\sstitle{Qualitative analysis} For the qualitative analysis of the matching explanation task, we have generated visualizations for two exemplary results derived from the KKI dataset. Specifically, \autoref{fig:quanlitative1} illustrates an isomorphic case, while \autoref{fig:quanlitative2} portrays a non-isomorphic case. A numerical label denotes each node within these figures, and nodes predicted to be aligned are color-coded uniformly. Our model demonstrates exceptional performance in the isomorphic case, accurately predicting all node mappings. Conversely, the model generates candidate mappings for all potential subgraphs within the pattern in the non-isomorphic case. For instance, in \autoref{fig:quanlitative2}, the subgraph $154-54-129-152$ can be transformed into an isomorphic subgraph of the target within the pattern by removing the edge $154-152$.

\begin{figure}[!ht]
    \centering
    \begin{subfigure}[b]{0.48\textwidth}
    \centering
    \begin{tikzpicture}[scale=1.7, font=\small, inner sep=1mm] 
      \draw
        (1.0, 0.0) node[circle,draw,fill={rgb:blue,115;green,158},text=white] (0){54}
        (0.866, 0.5) node[circle,draw] (1){87}
        (0.5, 0.866) node[circle,draw,fill={rgb:blue,178;green,114},text=white] (2){154}
        (-0.0, 1.0) node[circle,draw] (3){129}
        (-0.5, 0.866) node[circle,draw,fill={rgb:red,213;green,94},text=white] (4){155}
        (-0.866, 0.5) node[circle,draw] (5){156}
        (-1.0, -0.0) node[circle,draw,fill={rgb:red,47;green,37;blue,94},text=white] (6){19}
        (-0.866, -0.5) node[circle,draw] (7){72}
        (-0.5, -0.866) node[circle,draw] (8){152}
        (0.0, -1.0) node[circle,draw,fill={rgb:red,141;green,45;blue,57},text=white] (9){174}
        (0.5, -0.866) node[circle,draw] (10){74}
        (0.866, -0.5) node[circle,draw] (11){167};
      \begin{scope}[-]
        \draw (0) to (1);
        \draw (0) to (2);
        \draw (0) to (3);
        \draw (1) to (4);
        \draw (1) to (5);
        \draw (2) to (4);
        \draw (2) to (6);
        \draw (3) to (7);
        \draw (3) to (8);
        \draw (4) to (5);
        \draw (4) to (9);
        \draw (5) to (10);
        \draw (8) to (11);
      \end{scope}
    \end{tikzpicture}
    \begin{tikzpicture}[scale=0.8, font=\small, inner sep=1mm]
      \draw
        (1.0, 0.0) node[circle,draw,fill={rgb:blue,178;green,114},text=white] (0){154}
        (0.309, 0.951) node[circle,draw,fill={rgb:red,213;green,94},text=white] (1){155}
        (-0.809, 0.588) node[circle,draw,fill={rgb:blue,115;green,158},text=white] (2){54}
        (-0.809, -0.588) node[circle,draw,fill={rgb:red,141;green,45;blue,57},text=white] (3){174}
        (0.309, -0.951) node[circle,draw,fill={rgb:red,47;green,37;blue,94},text=white] (4){19};
      \begin{scope}[-]
        \draw (0) to (2);
        \draw (0) to (1);
        \draw (0) to (4);
        \draw (1) to (3);
      \end{scope}
    \end{tikzpicture}
    \caption{Subgraph isomorphism case. From up to down: The target graph and the isomorphic pattern graph}
    \label{fig:quanlitative1}
     \end{subfigure}
     \hfill
     \begin{subfigure}[b]{0.48\textwidth}
     \centering
    \begin{tikzpicture}[scale=1.7, font=\small, inner sep=1mm] 
      \draw
        (1.0, 0.0) node[circle,draw,fill={rgb:blue,115;green,158},text=white] (0){54}
        (0.866, 0.5) node[circle,draw] (1){87}
        (0.5, 0.866) node[circle,draw,fill={rgb:blue,178;green,114},text=white] (2){154}
        (-0.0, 1.0) node[circle,draw,fill={rgb:red,204;blue,167;green,121},text=white] (3){129}
        (-0.5, 0.866) node[circle,draw] (4){155}
        (-0.866, 0.5) node[circle,draw] (5){156}
        (-1.0, -0.0) node[circle,draw] (6){19}
        (-0.866, -0.5) node[circle,draw] (7){72}
        (-0.5, -0.866) node[circle,draw,fill={rgb:red,240;blue,66;green,228},text=white] (8){152}
        (0.0, -1.0) node[circle,draw,fill={rgb:red,141;green,45;blue,57},text=white] (9){174}
        (0.5, -0.866) node[circle,draw] (10){74}
        (0.866, -0.5) node[circle,draw] (11){167};
      \begin{scope}[-]
        \draw (0) to (1);
        \draw (0) to (2);
        \draw (0) to (3);
        \draw (1) to (4);
        \draw (1) to (5);
        \draw (2) to (4);
        \draw (2) to (6);
        \draw (3) to (7);
        \draw (3) to (8);
        \draw (4) to (5);
        \draw (4) to (9);
        \draw (5) to (10);
        \draw (8) to (11);
      \end{scope}
    \end{tikzpicture}
    \begin{tikzpicture}[scale=1, font=\small, inner sep=1mm]
      \draw
        (1.0, 0.0) node[circle,draw,fill={rgb:red,204;blue,167;green,121},text=white] (0){129} 
        (0.623, 0.782) node[circle,draw,fill={rgb:blue,115;green,158},text=white] (1){54} 
        (-0.223, 0.975) node[circle,draw,fill={rgb:blue,178;green,114},text=white] (2){154} 
        (-0.901, 0.434) node[circle,draw] (3){80}
        (-0.901, -0.434) node[circle,draw,fill={rgb:red,141;green,45;blue,57},text=white] (4){174} 
        (-0.223, -0.975) node[circle,draw] (5){71}
        (0.623, -0.782) node[circle,draw,fill={rgb:red,240;blue,66;green,228},text=white] (6){152}; 
      \begin{scope}[-]
        \draw (0) to (1);
        \draw (0) to (6);
        \draw (1) to (5);
        \draw (1) to (2);
        \draw (2) to (3);
        \draw (2) to (6);
        \draw (3) to (4);
      \end{scope}
    \end{tikzpicture}
    \caption{Subgraph non-isomorphism case. From up to down: The target graph and the non-isomorphic pattern graph}
    \label{fig:quanlitative2}
     \end{subfigure}
     \caption{Examples of isomorphism and non-isomorphism cases resulted from our model in the KKI dataset}
\end{figure}
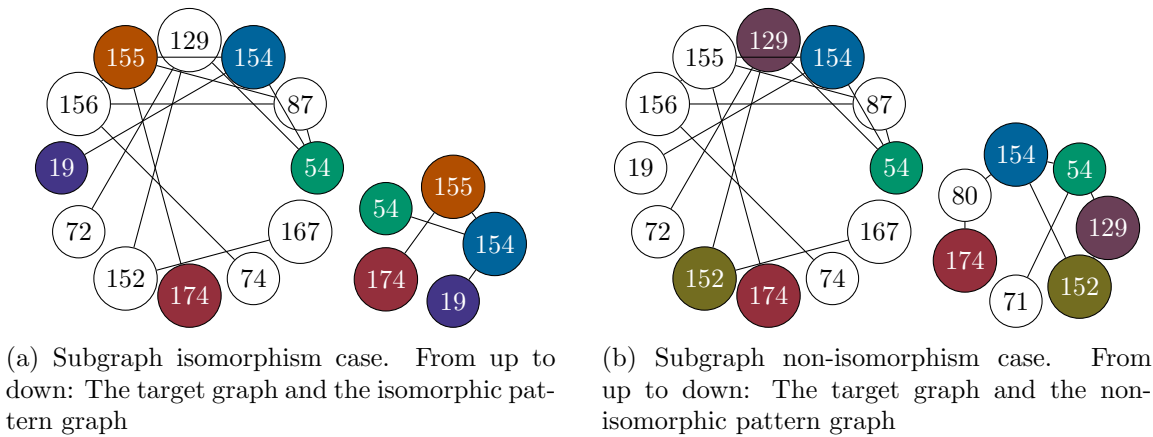

%% file: sections/exp_abtest.tex
\subsection{Ablation testing}

In this section, we conduct an ablation study to better understand the interactions between the components in \toolname, which is the answer for \textbf{RQ5}. We evaluate the performance of our proposed framework using the KKI dataset against several variants:

\begin{itemize}
\item \textbf{Model Architecture:}
\begin{itemize}
    \item \texttt{cross-only:} This configuration exclusively employs interconnections between the graph and subgraph.
    \item \texttt{intra-only:} This configuration solely relies on intra-connections within the graph and subgraph.
     \item \texttt{both:} This configuration combines intra- and inter-connections of the graph and subgraph.
\end{itemize}
\item \textbf{Multi-hop Mechanism:}
\begin{itemize}
 \item \texttt{1-hop:} Here, we replace the GLeMA layer with a standard 1-hop GAT layer.
 
\item \texttt{increasing-hop:} This configuration employs the GLeMA with a continuously increasing number of hops in the deeper layers ($K^{(L_G)} = L_G$).

\item \texttt{interleaved-hop:} This configuration uses the GLeMA with an interleaving-increasing number of hops. In this study, we use $K^{(L_G)} = 2L_G - 1$.
\end{itemize}
\end{itemize}

The results presented in~\autoref{tab:abtest} provide compelling evidence that our proposed model architecture strikes a balanced and optimal trade-off between performance and computational efficiency. These findings demonstrate that modelling intra-connections is crucial for achieving outstanding performance. Combining inter- and intra-connections enhances the model's ability to distinguish unaligned nodes, resulting in higher performance than using intra- or inter-connections in isolation. Additionally, increasing the number of hops enables the model to effectively capture the global structure of graphs, further boosting performance. However, a continuous increase in hops leads to a slowdown in the model. The interleaved-hop strategy is the most suitable option for maintaining high performance while reducing computational time.

\begin{table}[!ht]
 \caption{Impact of each model component}
  \label{tab:abtest}
  \centering
  \resizebox{\textwidth}{!}{
  \begin{tabular}{lcccccc}
    \toprule
     \textbf{Model} & \textbf{Runtime}$\downarrow$ & \textbf{ROC}$\uparrow$ & \textbf{PR}$\uparrow$ & \textbf{F1}$\uparrow$ & \textbf{Acc}$\uparrow$ & \textbf{MRR}$\uparrow$ \\
    \midrule
    \toolname cross-only 1-hop & 0.56 & 0.979 & 0.959 & 0.979 & 0.979 & \textbf{0.999}\\
    \toolname cross-only increasing-hop & 0.60 & 0.977 & 0.956 & 0.977 & 0.977 & 0.996 \\
    \toolname cross-only interleaved-hop & 0.42 & 0.978 & 0.958 & 0.978 & 0.978 & 0.997\\
    \toolname intra-only 1-hop & 0.49 & 0.611 & 0.578 & 0.485 & 0.612 & -\\
    \toolname intra-only increasing-hop & 0.40 & 0.628 & 0.593 & 0.515 & 0.628 & - \\
    \toolname intra-only interleaved-hop & 0.42 & 0.669 & 0.626 & 0.602 & 0.670 & - \\ 
    \toolname both 1-hop & 0.62 & 0.968 & 0.939 & 0.968 & 0.968 & \textbf{0.999} \\
    \toolname both increasing-hop & 0.70 & \textbf{0.980} & 0.963 & \textbf{0.980} & \textbf{0.980} & \textbf{0.999} \\
    \textbf{\toolname} & 0.51 & 0.979 & \textbf{0.964} & \textbf{0.980} & 0.979 & 0.998\\
       \bottomrule
  \end{tabular}
  }
\end{table}

%% file: sections/exp_generalisation.tex
\subsection{Generalisation analysis}
In this section, we conduct experiments to demonstrate the generalization capabilities of \toolname in out-of-distribution settings and answer to \textbf{RQ6}. In these settings, we utilize the model trained on one dataset to test the others with the same datasets as in the previous experiments. The testing results are presented in \autoref{tab:generalization}. Upon examination of \autoref{tab:generalization}, the following observations can be made:
\begin{itemize}
\item  A model trained on a dataset with a large $T_V$ exhibits generalization to datasets with smaller $T_V$ (trained on KKI and tested on DHFR, DBLP-v1, MSRC-21).
\item Conversely, a model trained on a dataset with fewer duplicated node graphs does not generalize well to datasets with more duplicated node graphs (No model trained on other datasets generalizes to COX2).
\item Furthermore, a model trained on dense graphs can generalize to datasets with sparser graphs (trained on MSRC-21 and tested on DHFR, DBLP-v1).
\end{itemize}
These results collectively demonstrate the strong generalization abilities of our model, especially when the model is trained on sufficiently large datasets.

\begin{table}[!ht]
     \caption{ROC AUC of out-distribution setting}
      \label{tab:generalization}
      \centering
    \begin{tabular}{lcccccc}
        \toprule
        \diagbox{\textbf{Train}}{\textbf{Test}} & \textbf{KKI} & \textbf{COX2} & \textbf{COX2\_MD} & \textbf{DHFR} & \textbf{DBLP-v1} & \textbf{MSRC-21} \\
        \midrule
        \textbf{KKI} & \textbf{0.979} & 0.634 & 0.499 & \cellcolor{blue!25}\textit{0.970} & \cellcolor{blue!25}\textit{0.923} & \cellcolor{blue!25}\textit{0.928}\\
        \textbf{COX2} & 0.500 & \textbf{0.983} & 0.500 & 0.500 & 0.501 & 0.500\\
        \textbf{COX2\_MD} & 0.534 & 0.412 & \textbf{0.986} & 0.499 & 0.565 & 0.497\\
        \textbf{DHFR} & 0.547 & 0.797 & 0.499 & \textbf{0.998} & 0.758 & 0.668\\
        \textbf{DBLP-v1} & 0.502 & \cellcolor{blue!25}\textit{0.883} & 0.491 & 0.689 & \textbf{0.966} & 0.505\\
        \textbf{MSRC-21} & \cellcolor{blue!25}\textit{0.863} & 0.539 & \cellcolor{blue!25}\textit{0.604} & 0.961 & 0.712 & \textbf{0.997} \\
         \bottomrule
    \end{tabular}
\end{table}







%% file: sections/exp_directed.tex
\subsection{Directed subgraph matching and explanation}
To answer the final \textbf{RQ7}, we assessed our \toolname using directed graphs. We utilized the same datasets as those in previous experiments to accomplish this. We converted all edges within these datasets into directed edges, designating the tail node as the one with a smaller label and the head node as the one with a more prominent label. Subsequently, we present the results of our evaluation within the context of both subgraph matching and matching explanation tasks in Table \ref{tab:directed_test}. The results in Table \ref{tab:directed_test} demonstrate the effectiveness of our proposed method, regardless of whether the graph is directed or not.

\begin{table}[!ht]
 \caption{Performance of \toolname directed subgraph matching and matching explanation}
  \label{tab:directed_test}
  \centering
\resizebox{1\textwidth}{!}{  
  \begin{tabular}{lcccccccc}
    \toprule
     \textbf{Dataset} & \textbf{ROC}$\uparrow$ & \textbf{PR}$\uparrow$ & \textbf{F1}$\uparrow$ & \textbf{Acc}$\uparrow$ & \textbf{Top-1}$\uparrow$ & \textbf{Top-2}$\uparrow$ & \textbf{Top-10}$\uparrow$ & \textbf{MRR}$\uparrow$ \\
    \midrule
    KKI & 0.9750 & 0.9532 & 0.9754 & 0.9749 & 0.9962 & 0.9998 & 0.9998 & 0.9980\\
    COX2 & 0.9473 & 0.9085 & 0.9494 & 0.9474 & 0.1030 & 0.3962 & 0.6404 & 0.2611\\
    COX2\_MD & 0.9895 & 0.9797 & 0.9896 & 0.9895 & 0.9999 & 0.9999 & 0.9999 & 0.9999\\
    DHFR & 0.9693 & 0.9441 & 0.9702 & 0.9694 & 0.9999 & 0.9999 & 0.9999 & 0.9999\\
    DBLP-v1 & 0.9603 & 0.9406 & 0.9604 & 0.9603 & 0.7454 & 0.99607 & 0.9994 & 0.8668\\
    MSRC-21 & 0.9880 & 0.9773 & 0.9881 & 0.9880 & 0.9999 & 0.9999 & 0.9999 & 0.9999\\
    \bottomrule
  \end{tabular}
}
\end{table}

%% file: sections/related.tex
\section{Related works}
\label{sec:related}

Subgraph matching is addressed through two distinct settings within research communities. This section covers approaches for \emph{non-induced subgraph matching} (\autoref{sec:non-induced}) and \emph{induced subgraph matching} (\autoref{sec:induced}). Furthermore, since we employ a neural-based approach capable of explaining subgraph matching, we also examine related techniques for \emph{neural subgraph matching and explanation} (\autoref{sec:neural}).

\subsection{Non-induced subraph matching}
\label{sec:non-induced}
The concept of non-induced subgraph matching allows the pattern, denoted as $\mathcal{P}$, to serve as a partial embedding within a subgraph, denoted as $\mathcal{S}$, of the larger data graph, $G$. In other words, while an edge exists, denoted as $e \in E_{\mathcal{S}}$, it may not necessarily be present in $E_{\mathcal{P}}$. This particular setting holds significant utility in various aspects of data management, including tasks like graph indexing, graph similarity search, and graph retrieval~\citep{roy2022interpretable}. Recently, there has been a surge of interest in the community surrounding the support for explainable subgraph matching, primarily driven by the introduction of efficient approaches~\citep{kim2021versatile}. These developments have garnered attention for their potential to enhance the interpretability and applicability of subgraph matching techniques. However, our current work primarily focuses on exploring induced subgraph matching, a topic that will be delved into further in the subsequent section. This approach involves identifying patterns where all the edges in the pattern must also exist in the data graph, providing a more rigorous condition for matching.

\subsection{Induced subraph matching}
\label{sec:induced}
The induced subgraph matching problem has been proven NP-complete~\cite {lewis1983michael}. Various algorithms~\citep{shang2008taming, he2008graphs, han2013turboiso, bi2016efficient, bhattarai2019ceci, han2019efficient} have been proposed to address this challenge, focusing on generating effective matching orders and designing robust filtering strategies to reduce the number of candidates in the data graph. For instance, QuickSI~\citep{shang2008taming} introduces the \emph{infrequent-edge-first ordering} technique. This approach arranges the edges of the query graph in ascending order based on their frequency of appearance in the data graph. In contrast, GraphQL~\citep{he2008graphs} employs the \emph{left-deep-join ordering} strategy, conceptualizing the enumeration procedure as a joint problem. TurboIso~\citep{han2013turboiso} and CFL~\citep{bi2016efficient} advocate for the path-based ordering method, which entails decomposing the query graph into several paths and ordering them according to the estimated number of embeddings for each path. In addition to these ordering strategies, state-of-the-art algorithms like TurboIso~\citep{han2013turboiso}, CECI~\citep{bhattarai2019ceci}, and CFL~\citep{bi2016efficient} adopt a \emph{tree-based framework}. This framework constructs a lightweight, tree-structured index to minimize the number of candidates. Subsequently, it enumerates all matches based on this index rather than the original data graph. While these techniques have undeniably propelled significant progress, we have identified several inherent issues.

\subsection{Neural Subgraph Matching and Explanation}
\label{sec:neural}
The initial work~\citep{scarselli2008graph} attempted to assess the feasibility of Graph Neural Networks (GNNs) in subgraph matching, which was validated with small-scale subgraphs. This underscored the potential superiority of GNNs over feedforward neural networks. With recent advancements in GNNs~\citep{kipf2016semi, hamilton2017inductive,xu2018powerful}, contemporary subgraph matching techniques~\citep{bai2018convolutional,zhang2019deep,lou2020neural} have achieved state-of-the-art results in terms of efficiency. However, a need remains to explain the one-to-one correspondence between the graph pattern and the data graph, which hinders its direct application in downstream tasks like subgraph isomorphism testing. Recent studies have delved into the interpretability of GNNs~\citep{yuan2020xgnn,vu2020pgm,wu2023explaining} through a model-intrinsic perspective. They aim to explain which features of the data graph contribute to the GNN's performance on specific tasks, including subgraph matching~\citep{wu2023explaining}. However, it is worth noting that this line of inquiry is separate from our work, and we need to address this aspect of GNN interpretability in this article.

One closely related problem to our work is inexact matching. Despite employing an approximate neural-based approach, our validated outcomes are computed only when an exact pattern is matched. These allowable inexact approaches encompass a range of techniques, such as structural equivalance~\citep{yang2023structural}, inexact matching~\citep{sun2020subgraph}, and knowledge graph~\citep{sun2022subgraph}. This article will not delve into this problem, as we aim to reserve its consideration for a future study when we extend our framework to address these issues.

%% file: sections/conclusion.tex
\section{Conclusion}
\label{sec:conclusion}
In this study, we proposed a novel framework called xNeuSM for explainable neural subgraph matching. xNeuSM aims to address the limitations of previous neural-based approaches which lack interpretability while achieving superior performance compared to existing approximate and exact algorithms.
We introduced key contributions including the Graph Learnable Multi-hop Attention Networks and a multi-task learning framework to jointly optimize subgraph matching and matching explanation tasks. Theoretical justifications were provided to analyse the approximation error of multi-hop attention and prove the correctness of learning node-specific attention decays.

Extensive experimental evaluations on real-world datasets demonstrated that xNeuSM achieves substantial improvements over state-of-the-art techniques in both runtime and accuracy for subgraph matching. Its capability is further manifested in the precise identification of node mappings, as evidenced by the matching explanation results. Further ablation studies validated the effectiveness of individual components in xNeuSM's architecture.
This work also explored the generalisability, scalability and applicability of xNeuSM to directed graphs and inductive settings. Overall, xNeuSM achieves the dual objectives of superior performance and interpretability, making it a practical solution for a wide range of real-world tasks involving subgraph matching and pattern analysis in large graphs.

There are several promising directions for future work. Firstly, xNeuSM can be extended to handle inexact subgraph matching. Secondly, integrating advanced GNN modules may further boost xNeuSM's representation power. Lastly, applications of xNeuSM to real-world domains like drug discovery and network alignment could be explored. In summary, xNeuSM presents a significant step towards building interpretable and scalable neural solutions for graph-related problems.

%% file: sections/appendices.tex
\appendix

\section{Theoretical justifications}

\subsection{Multi-hop attention approximation error}
\label{sec:approx_err}
In this section, we will justify multi-hop attention approximation error and guide on choosing the appropriate number of approximate hop $K$. Firstly, we decompose $\mathbi{Z}^{(K)}$ as following.

\begin{equation}
    \begin{aligned}
        \mathbi{Z}^{(K)} =&\ (1-\alpha)^K(\mathcal{A}^{(1)})^K \mathbi{X'} + \alpha (1-\alpha)^{K-1}(\mathcal{A}^{(1)})^{K-1} \mathbi{X'} \\
        &+ \dots + \alpha (1-\alpha)(\mathcal{A}^{(1)}) \mathbi{X'} + \alpha \mathbi{X'}
    \end{aligned}
\end{equation}
Then, we get:
\begin{equation}
    \begin{aligned}
        \mathbi{Z}^{(K)} (\mathbi{X'})^{-1} =&\ (1-\alpha)^K(\mathcal{A}^{(1)})^K + \alpha (1-\alpha)^{K-1}(\mathcal{A}^{(1)})^{K-1} \\
        &+ \dots + \alpha (1-\alpha)(\mathcal{A}^{(1)}) + \alpha
    \end{aligned}
\end{equation}

By Proposition 1, we have 
\begin{equation}
    lim_{K \to \infty} \mathbi{Z}^{(K)} = \mathcal{A}\mathbi{X'}, 
\end{equation}
then
\begin{equation}
    lim_{K \to \infty} \mathbi{Z}^{(K)} (\mathbi{X'})^{-1} = \mathcal{A}.
\end{equation}

Now, let us consider the difference between attention diffusion matrix $\mathcal{A}$ and its approximate form $\mathbi{Z}^{(K)} (\mathbi{X'})^{-1}$.

\begin{equation}
\label{equ:err_azx}
\begin{aligned}
    \mathcal{A} - \mathbi{Z}^{(K)} (\mathbi{X'})^{-1} &= \sum_{k=0}^{\infty} \alpha (1-\alpha)^k (\mathcal{A}^{(1)})^{k} - (1-\alpha)^K (\mathcal{A}^{(1)})^K - \sum_{k=0}^{K-1} \alpha (1-\alpha)^k (\mathcal{A}^{(1)})^{k}\\
    &= \sum_{k=K}^{\infty} \alpha (1-\alpha)^k (\mathcal{A}^{(1)})^{k} - (1-\alpha)^K (\mathcal{A}^{(1)})^K\\
    &\leq \sum_{k=K}^{\infty} \alpha (1-\alpha)^k (\mathcal{A}^{(1)})^{k} - \alpha(1-\alpha)^K (\mathcal{A}^{(1)})^K \text{ by } \alpha, a_{ij}^{(1)} \in (0,1)\\
    &\leq \sum_{k=K+1}^{\infty} \alpha (1-\alpha)^k (\mathcal{A}^{(1)})^{k}
\end{aligned}
\end{equation}

We also have $a_{ij}^{(1)} \in (0,1)$ so that 
\begin{equation}
    (\mathcal{A}^{(1)})^{k} \leq (\mathcal{A}^{(1)})^{k-1}.
\end{equation}
By this property, \eqref{equ:err_azx} can be derived as follows.
\begin{equation}
   \mathbi{P} = \mathcal{A} - \mathbi{Z}^{(K)} (\mathbi{X'})^{-1} \leq \left(\sum_{k=K+1}^{\infty} \alpha (1-\alpha)^k \right) \mathcal{A}^{(1)}
\end{equation}
It is easy to observe that $\mathbi{P} \in \mathbb{R}^{|V| \times |V|}$. Then, we can define the average difference between the exact and approximate attention diffusion matrix as
\begin{equation}
    \begin{aligned}
        Err &= \frac{1}{|V|^2} \sum_{i,j} \mathbi{P}_{ij}\\
        &\leq \frac{1}{|V|^2} \sum_{i,j} \left(\left(\sum_{k=K+1}^{\infty} \alpha (1-\alpha)^k \right) a_{ij}^{(1)} \right)\\
        &\leq \sum_{k=K+1}^{\infty} \alpha (1-\alpha)^k \text{ by } a_{ij}^{(1)} \in (0,1)\\
        &\leq \alpha \sum_{k=K+1}^{\infty} (1-\alpha)^k \leq \alpha \frac{(1-\alpha)^{K+1}}{1 - (1-\alpha)}\\
        &\leq (1-\alpha)^{K+1}
    \end{aligned}
\end{equation}
It is readily apparent that when the error ($Err$) is constrained by the condition $Err \leq (1-\alpha)^{K+1}$, the selection of $K$ from the set ${3, \ldots, 10}$ yields errors that are consistently below the threshold of 0.3 for values of $\alpha$ greater than or equal to 0.3. To further illustrate this observation, we have presented a graphical representation of the error as a function of $K$ in Figure \ref{fig:alpha_by_k}.
\begin{figure}[!ht]
    \centering
    \includegraphics{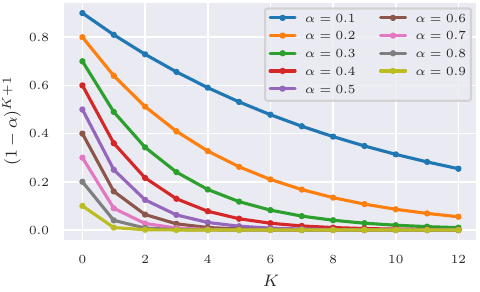}
    \caption{Maximal approximation error of each attention coefficient}
    \label{fig:alpha_by_k}
\end{figure}

\subsection{Correctness of separated teleport probability for each node}
\label{sec:separated_teleport}
This section establishes the validity of employing distinct teleport probabilities, denoted as $\beta_v \in (0,1)$ for each node $v^{th}$, without violating Proposition 1. This variation in teleport probabilities gives rise to varying decay factors for each node, denoted as $\eta_v \in \mathbb{R}$. Specifically, we have:
\begin{equation}
   (\eta_v)_k = \beta_v (1-\beta_v)^k > 0
\end{equation}
This results in the important property:
\begin{equation}
    \forall v \in |V|, \sum_{k=0}^{\infty} (\eta_v)_k = 1.
\end{equation}
With this property, we can generalize Equation \ref{equ:attn_diff_mth} as follows:
\begin{equation}
    \begin{cases}
    (\mathcal{A}^{(1)})^{0} &= \mathbi{I}\\
    \mathcal{A}_{\eta} &= \sum_{k=0}^{\infty} \eta_k (\mathcal{A}^{(1)})^{k}\ \text{where}\ \eta_k=\{(\eta_v)_k\}_{v=1}^{|V|}.
    \end{cases}
\end{equation}
In the case where $\beta_i = \beta_j$ for $i,j \in \overline{1, |V|}$, $\mathcal{A}_{\eta}$ reduces to $\mathcal{A}$.

Using the same technique, we can approximate $\mathcal{A}_{\eta}\mathbi{X'}$ as:
\begin{equation}
    \begin{cases}
        \mathbi{Z}^{(0)} &= \mathbi{X{'}}\\
        \mathbi{Z}^{(k)}_{\beta} &= (\vec{1} - \beta) \mathcal{A}^{(1)} \mathbi{Z}^{(k-1)} + \beta \mathbi{Z}^{(0)},\ k = \overline{1,K}
    \end{cases}
\end{equation}

\begin{proposition} 
$lim_{K \to \infty} \mathbi{Z}^{(K)}_{\beta} = \mathcal{A}_{\eta}\mathbi{X'}$
\end{proposition}

Now, we proceed to prove Proposition 2 as follows.
Firstly, we decompose all elements of $\mathbi{Z}^{(K)}_{\beta}$:
\begin{equation}
    \begin{aligned}
        \mathbi{Z}^{(K)}_{\beta} &= \underbrace{(\vec{1} - \beta) \mathcal{A}^{(1)}\dots }_{K\text{ times}} \mathbi{X'} + \underbrace{(\vec{1} - \beta) \mathcal{A}^{(1)}\dots}_{K-1\text{ times}} \beta\mathbi{X'} \\
        &+ \underbrace{(\vec{1} - \beta) \mathcal{A}^{(1)}\dots}_{K-2\text{ times}} \beta\mathbi{X'} + \dots + (\vec{1} - \beta) \mathcal{A}^{(1)}\beta\mathbi{X'} + \beta\mathbi{X'} 
    \end{aligned}
\end{equation}
We also have:
\begin{equation}
    \begin{aligned}
        \mathcal{A}_{\eta}\mathbi{X'} &= \left(\sum_{k=0}^{\infty} \eta_k (\mathcal{A}^{(1)})^{k}\right) \mathbi{X'}\\
        &= \eta_0\mathbi{X'} + \eta_1\mathcal{A}^{(1)}\mathbi{X'} + \eta_2(\mathcal{A}^{(1)})^2\mathbi{X'} + \dots\\
        &= \beta\mathbi{X'} + \beta(\vec{1} - \beta)\mathcal{A}^{(1)}\mathbi{X'} + \beta(\vec{1} - \beta)^2(\mathcal{A}^{(1)})^2\mathbi{X'} + \dots\\
    \end{aligned}
\end{equation}
To prove Proposition 2, we need to prove these lemmas.
\begin{lemma}
\[
\underbrace{(\vec{1} - \beta) \mathcal{A}^{(1)}\dots}_{k\text{ times}} \beta\mathbi{X'} = \beta(\vec{1} - \beta)^k(\mathcal{A}^{(1)})^k\mathbi{X'}
\]
\end{lemma}
\begin{proof}
Due to the commutative property of row-wise multiplication between a vector and a matrix, we can write:
\begin{equation}
    (\vec{1} - \beta)\mathcal{A}^{(1)} = \mathcal{A}^{(1)}(\vec{1} - \beta).
\end{equation}
This leads to:
\begin{equation}
    \begin{aligned}
    \underbrace{(\vec{1} - \beta) \mathcal{A}^{(1)}\dots}_{k\text{ times}} \beta\mathbi{X'} &= \underbrace{(\vec{1} - \beta)(\vec{1} - \beta)\dots}_{k\text{ times}} + \underbrace{\mathcal{A}^{(1)} \mathcal{A}^{(1)} \dots}_{k\text{ times}} \beta\mathbi{X'}\\
    &= (\vec{1} - \beta)^k (\mathcal{A}^{(1)})^k \beta\mathbi{X'}\\
    &= \beta(\vec{1} - \beta)^k (\mathcal{A}^{(1)})^k \mathbi{X'}.
    \end{aligned}
\end{equation}
\end{proof}

\begin{lemma}
\[\lim_{K \to \infty} \underbrace{(\vec{1} - \beta) \mathcal{A}^{(1)}\dots }_{K\text{ times}} \mathbi{X'}= \vec{0}.\]
\end{lemma}
\begin{proof}
Because $\beta_v \in (0,1)$, it follows that $(1-\beta_v) \in (0,1)$. This implies that: 
\[
lim_{K \to \infty} (1-\beta_v)^K = 0.
\]
We also have:
\begin{equation}
    (\vec{1} - \beta)^K = \{(1-\beta_v)^K\}_{v=1}^{|V|},
\end{equation}
so that 
\begin{equation}
    \lim_{K \to \infty}(\vec{1} - \beta)^K = \{\lim_{K \to \infty}(1-\beta_v)^K\}_{v=1}^{|V|} = \vec{0}.
\end{equation}
With the above properties, we conclude:
\begin{equation}
    \begin{aligned}
    \lim_{K \to \infty} \underbrace{(\vec{1} - \beta) \mathcal{A}^{(1)}\dots }_{K\text{ times}} \mathbi{X'} &= \lim_{K \to \infty} (\vec{1} - \beta)^K (\mathcal{A}^{(1)})^K \mathbi{X'}\\
    &= (\lim_{K \to \infty} (\vec{1} - \beta)^K) (\lim_{K \to \infty} (\mathcal{A}^{(1)})^K \mathbi{X'})\\
    &= \vec{0} (\lim_{K \to \infty} (\mathcal{A}^{(1)})^K \mathbi{X'})\\
    &= \vec{0}.
    \end{aligned}
\end{equation}
\end{proof}

Thus, Proposition 2 is proven.

\section{Reproducibility}
All experiments were conducted on a machine equipped with a 32-core CPU, 128GB of RAM, and an NVIDIA 2080Ti GPU with 12GB of memory. Our implementation can be found at \url{https://github.com/martinakaduc/xNeuSM.git}.

\section{Hyperparameter settings}
To support the community to reproduce \toolname, we provide our hyperparameter settings for training our \toolname model in Table \ref{tab:hyperparameter}.

\begin{table}[!ht]
    \centering
    \caption{Hyperparameter settings for our \toolname model}
    \label{tab:hyperparameter}
    \begin{tabular}{l|c}
    \toprule
    \textbf{Hyperparameter(s)} & \textbf{Value(s)} \\
    \midrule
     Learning rate & $10^{-4}$ \\
     Optimizer & Adam \\
     Number of epochs & 30 \\
     Number of GMA layers & 4 \\
     Number of hops & 1 3 5 7 \\
     Hidden dimension in GMA layer & 140 \\
     Number of attention head  & 1 \\
     Number of FC layers & 4 \\
     Hidden dimension in FC layer & 128 \\
     $\lambda$ & 1.0 \\
     \bottomrule
    \end{tabular}    
\end{table}